\documentclass{article} 
\usepackage{iclr2026_conference,times}

\usepackage{amsmath,amsfonts,bm}

\def\eqref#1{equation~\ref{#1}}

\def\1{\bm{1}}

\DeclareMathAlphabet{\mathsfit}{\encodingdefault}{\sfdefault}{m}{sl}
\SetMathAlphabet{\mathsfit}{bold}{\encodingdefault}{\sfdefault}{bx}{n}

\newcommand{\E}{\mathbb{E}}

\DeclareMathOperator*{\argmax}{arg\,max}

\usepackage{hyperref}
\usepackage{url}
\usepackage{booktabs}       
\usepackage{amsfonts}       
\usepackage{nicefrac}       
\usepackage{microtype}      
\usepackage{xcolor}         
\usepackage{bm}
\usepackage{amsmath,amssymb}
\usepackage{amsthm}
\usepackage{algorithm}
\usepackage{algpseudocode}
\usepackage{bbm}
\usepackage{graphicx}
\usepackage{array}
\usepackage{subcaption}
\usepackage{comment}
\usepackage{enumitem}
\usepackage{wrapfig}

\newcommand{\Real}{\mathbb R}

\newcommand{\CalO}{\mathcal O}
\newcommand{\CalT}{\mathcal T}
\newcommand{\CalR}{\mathcal R}
\newcommand{\CalH}{\mathcal H}
\newcommand{\CalB}{\mathcal B}
\newcommand{\CalX}{\mathcal X}
\newcommand{\CalP}{\mathcal P}
\newcommand{\CalD}{\mathcal D}

\newcommand{\CalS}{\mathcal S}
\newcommand{\CalA}{\mathcal A}

\newcommand{\CalJ}{\mathcal J}
\newcommand{\CalN}{\mathcal N}
\newcommand{\pr}{\mathbb P}
\newcommand{\BX}{\bold X}
\newcommand{\BC}{\bold C}

\newcommand{\Bx}{\bold x}

\newcommand{\Br}{\bold r}

\newcommand{\Bomega}{\boldsymbol{\omega}}
\newcommand{\Btheta}{\boldsymbol{\theta}}
\newcommand{\Bvaresilon}{\boldsymbol{\varepsilon}}

\newcommand{\Bb}{\bold b}
\newcommand{\BK}{\bold K}

\newtheorem{theorem}{Theorem}
\newtheorem{lemma}[theorem]{Lemma} 
\newtheorem{proposition}[theorem]{Proposition} 

\newtheorem{corollary}[theorem]{Corollary}
\newtheorem{definition}[theorem]{Definition}

\newtheorem{assumption}[theorem]{Assumption}

\renewcommand{\headrulewidth}{0pt}  
\renewcommand{\footrulewidth}{0pt}  

\title{Sampling Complexity of TD and PPO in RKHS}

\author{
Lu Zou\textsuperscript{1} \quad
Wendi Ren\textsuperscript{2} \quad
Weizhong Zhang\textsuperscript{3} \quad
 Liang Ding\textsuperscript{3} \quad
 Shuang Li \textsuperscript{2}
\\
\textsuperscript{1} School of Management, Shenzhen Polytechnic University \\
\textsuperscript{2} School of Data Science, The Chinese University of Hong Kong (Shenzhen)\\
\textsuperscript{3} School of Data Science, Fudan University
}

\iclrfinalcopy 
\begin{document}

\maketitle

\begin{abstract}
We revisit Proximal Policy Optimization (PPO) from a function-space perspective. 
Our analysis decouples policy evaluation and improvement in a reproducing kernel Hilbert space (RKHS): 
(i)  A kernelized temporal-difference (TD) critic performs efficient RKHS-gradient updates using only one-step state–action transition samples.
(ii) a KL-regularized, natural-gradient policy step exponentiates the evaluated action-value, recovering a PPO/TRPO-style proximal update in continuous state-action spaces. 
We provide non-asymptotic, instance-adaptive guarantees whose rates depend on RKHS entropy, unifying tabular, linear, Sobolev, Gaussian, and Neural Tangent Kernel (NTK) regimes, and we derive a sampling rule for the proximal update that ensures the optimal $k^{-1/2}$  convergence rate for stochastic optimization.
Empirically, the theory-aligned schedule improves stability and sample efficiency on common control tasks (e.g., CartPole, Acrobot), while our TD-based critic attains favorable throughput versus a GAE baseline. 
Altogether, our results place PPO on a firmer theoretical footing beyond finite-dimensional assumptions and clarify when RKHS-proximal updates with kernel-TD critics yield global policy improvement with practical efficiency.
\end{abstract}

\section{Introduction}
Policy-gradient and trust-region methods (e.g., natural policy gradient (NPG) \citep{kakade2001natural}, trust-region policy optimization (TRPO) \citep{schulman2015trust}, proximal policy optimization (PPO) \citep{schulman2017proximal}, and actor–critic (AC) methods~\citep{konda1999actor}), when coupled with temporal-difference (TD) critics, are among the most effective tools in modern RL for large, continuous control, thanks to their compatibility with expressive function approximators and stable improvement steps. 
Yet despite impressive empirical success, our theoretical understanding of global convergence for these algorithms under expressive function approximation remains fragmented across settings. 
In particular, a central problem is to design an algorithm that (i) performs policy evaluation with controlled statistical error when the action–value function lies in a rich function class, and (ii) couples this evaluation step with a policy improvement update that provably ascends toward the optimal policy.

Existing analyses often establish convergence only in tabular/linear regimes or under strong realizability and concentrability conditions, see, e.g., \citep{agarwal2020optimality,bhandari2024global}. With nonlinear or nonparametric critics,  guarantees frequently rely on idealized, exact value/advantage estimates. Moreover, many policy-improvement bounds treat expectation terms as if computed without sampling noise, leaving the per-iteration data requirements for ensured improvement unspecified.

For TD learning \citep{sutton1988learning,maei2009convergent,bhandari2018finite}, linear TD is well understood, with asymptotic convergence and finite-time error bounds. By contrast, theory under nonlinear approximation is thinner. A key advance is \citep{cai2019neural}, which establishes finite-sample convergence of neural TD under overparameterized networks with discrete actions (and continuous states) at a sublinear rate. On the policy-optimization side, the strongest results remain in tabular or linear settings; with expressive function classes, guarantees weaken. Notably, \citep{liu2019neural} prove nonasymptotic global convergence of mirror descent policy optimization for two-layer overparameterized neural policies in continuous-state, discrete-action problems under the neural tangent regime.

In this paper, we take a function-space approach and optimize policies in a RKHS, which encompasses tabular/linear models, Sobolev classes, Gaussian kernels, and wide neural networks through their neural tangent kernels. We study policy evaluation and improvement using gradient-based updates: (i) a kernel TD critic--distinct from least-squares TD--that implements an RKHS-gradient iteration acting as an implicit preconditioner and avoids cubic-time matrix inversions, and (ii) a KL-regularized functional proximal step for policy improvement, implementable in continuous action spaces, where the policy is updated by exponentiating the evaluator’s value estimate.
Our main contributions are:

\begin{itemize}[leftmargin=2em]
    \item We introduce a kernel, gradient-based TD evaluator in an RKHS that acts as an implicit preconditioner and attains geometric convergence without costly matrix inversions.
    The evaluator uses one-step samples (no trajectory rollouts) and comes with non-asymptotic TD-error bounds that match the minimax rate (up to logarithms).
    
    \item We design a KL-regularized proximal update implementable in continuous action spaces. We explicitly quantify the per-iteration sample size needed to achieve the intended improvement, addressing a common gap where policy expectations are treated as exact or left unspecified.

    \item Because kernel gradient descent mirrors the NTK dynamics of wide networks trained by gradient descent, our RKHS analysis directly informs neural critics/actors in the corresponding regimes. Experiments on continuous-control benchmarks exhibit the trends predicted by our theory.
    
\end{itemize}



\section{Related Studies}

TD method \citep{sutton1988learning} is one of the most commonly used for policy evaluation. The convergence of linear TD has been extensively studied, with finite-time error bounds established in recent works \citep{bhandari2018finite,lakshminarayanan2018linear,srikant2019finite}. In contrast, the behavior of TD with nonlinear function approximation remains less understood. A notable advance is due to~\cite{cai2019neural}, who provided the first finite-sample analysis of neural TD, proving sublinear convergence under an overparameterized network; see also~\cite{brandfonbrener2019geometric,agazzi2019temporal} for related results in tabular settings.

Policy optimization has also been extensively studied, with algorithms including policy gradient (PG) \citep{sutton1999policy,baxter2000direct}, natural policy gradient (NPG) \citep{kakade2001natural}, trust-region policy optimization (TRPO) \citep{schulman2015trust}, proximal policy optimization (PPO) \citep{schulman2017proximal}, and actor–critic (AC) methods~\citep{konda1999actor}. Among these, NPG has been analyzed most thoroughly. Its convergence is well understood in the tabular setting, while function approximation presents additional challenges~\citep{bhandari2024global,cen2022fast,mei2020global}. For linear approximation, \citet{agarwal2020optimality} established global convergence in tabular setting and restricted class of parametric
policies, \citet{agarwal2021theory} derived finite-sample rates for unregularized NPG with softmax parameterization. In more general settings, \citet{zhang2020global} obtained only local guarantees, and \citet{cayci2024convergence} provided sharp nonasymptotic bounds for entropy-regularized NPG. Broader classes of function approximation, including neural networks, have also been considered~\citep{wang2019neural,liu2019neural}. TRPO and PPO have likewise received theoretical attention: \citet{neu2017unified} and \citet{shani2020adaptive} analyzed TRPO, while \citet{liu2019neural} and \citet{cai2020provably} studied PPO and its variant.

Our work is closely related to kernel methods \citep{hofmann2008kernel,zhou2008derivative,cho2009kernel}, which is one of the commonly used approaches in policy learning, e.g., \cite{bagnell2003policy,bethke2008kernel,grunewalder2012modelling,feng2020accountable,koppel2020policy}. Early work e.g., \cite{ormoneit2002kernel,munos2008finite}, established consistency results for non-parametric value function approximation.  More refined analyses were later provided by \cite{farahmand2016regularized}, who studied the convergence of the Bellman contraction mapping under the non-parametric setting, and by \cite{duan2024optimal}, who proposed a regularized kernel-based LSTD estimator in RKHS for Markov chains independent of actions, deriving non-asymptotic error bound and establishing matching minimax lower bound.


\section{Preliminaries}

\textbf{\textit{ Markov Decision Process.}} We consider the MDP  $(\CalS,\CalA,P,r,\gamma)$, where $\CalS\subset\Real^{d_2}$ and $\CalA\subset\Real^{d_a}$ are compact and convex sets,  $P:\CalS\times\CalA \to \CalP(\CalS)$ is the transition kernel for any state-action pair $(s,a)\in\CalS\times\CalA$ to the space of distribution on $\CalS$ denoted as $\CalP(\CalS)$, $r:\CalS\times\CalA\to\Real$ is the reward function, and $\gamma\in [0,1]$ is the discount factor. These yields a Markov chain $(s_0,a_0,s_1,a_1,\cdots)$.

The performance of a policy $\pi:\CalS\to\CalP(\CalA)$ is evaluated by the following value function:
\begin{equation}
    V^{\pi}(s)=\E[\sum_{t=0}^\infty\gamma^tr(s_t,a_t)\big|s_0=s,a_t\sim\pi(\cdot|s_t),s_{t+1}\sim P(\cdot|s_t,a_t)],\label{eq:value_func}
\end{equation}
and the action-value function (Q-function):
\begin{align}
Q^\pi(s,a)=&\E[\sum_{t=0}^\infty\gamma^tr(s_t,a_t)\big|s_0=s,a_0=a,a_t\sim\pi(\cdot|s_t),s_{t+1}\sim P(\cdot|s_t,a_t)].\label{eq:Q-func}
\end{align}
The corresponding advantage function $A^\pi(s,a)$ is define as $A^\pi(s,a)=Q^\pi(s,a)-V^\pi(s)$. In this study, we will focus on learning the Q-function. To this end, we define the Bellman evaluation operator:
\begin{equation}
\label{eq:bellman}
    \CalT[Q^{\pi}](s,a)=r(s,a)+\gamma \E_{a'\sim \pi(\cdot|s'),s'\sim P(s'|s,a)}[Q^\pi(s',a')],
\end{equation}
for which $Q^\pi$ is the fixed point of the operator: $Q^{\pi}=\CalT Q$.

We let  $\sigma^\pi_\mu(s,a)=\pi(a|s)\mu(s)$ denote the action-state distribution associated to policy $\pi$. Let $\nu^*$ denote the stationary state distribution of the Markov chain given the optimal policy $\pi^*$ and let $\sigma^*(s,a)=\pi^*(a|s)\nu^*(s)$ denote the  stationary state-action distribution. When given a policy $\pi$, we first sample $n$ initial samples from a pre-determined initial distribution $\mu_0$. We have the following assumption on $\sigma_0^{\pi}=\pi\mu_0$ and $\sigma^*$.

\begin{assumption}
\label{assump:distribution_bounded}
    Distributions $C_1\geq \sigma^*\geq c_1$ and  $C_2\geq\sigma^\pi_0\geq c_2$ are bounded.
\end{assumption}

Assumption~\ref{assump:distribution_bounded} implied that $\sigma^*$ and $\sigma^\pi_0$ are neither overly concentrated nor sparse. This condition is crucial because such distributions would cause numerical difficulties during their evaluation.

\textbf{\textit{RKHS.}} Define $\omega=(s,a)\in\CalS\times\CalA$. We assume the Q-function $Q$ associated with a policy $\pi$ lies in $\mathcal{H}(\mathcal{S}\times\mathcal{A})$ where $\mathcal{H}$ is a RKHS induced by a symmetric positive definite kernel $K : (\CalS\times\mathcal{A})\times(\CalS\times\mathcal{A}) \to \mathbb{R}$. Using kernel $K$, we define the linear space of functions on $\mathcal{S}\times\mathcal{A}$ as follows:
\[\hat{\CalH}=\{\sum_{i=1}^nb_i K(\omega_i,\cdot), b_i\in \Real, \omega_i\in\CalS\times\CalA\}\]
and it is equipped this space with the bilinear form
\[\langle \sum_{i=1}^nb_i K(\omega_i,\cdot),\sum_{i=1}^nc_i K(\omega_i,\cdot)\rangle_K=\sum_{i,j=1}^n b_iK(\omega_i,\omega_j)c_j .\]
Without loss of generality, we assume that $\max_{\omega}K(\omega,\omega)<\infty$ is bounded. The RKHS $\CalH(\CalS\times\CalA)$ induced by $K$ is defined as the closure of $\hat{\CalH}$ under the inner product $\langle\cdot,\cdot\rangle_K$ and  $\CalH(\CalS\times\CalA)$ is equipped  with norm $\|\cdot\|_{\CalH(\CalS\times\CalA)}$ induced by $\langle\cdot,\cdot\rangle_K$. 

Many functional spaces can be represented within the RKHS framework. For example, deep neural networks, Sobolev spaces, the Euclidean space, and discrete set can all be formulated as RKHSs. For further details, please refer  to \cite{adams2003sobolev,wendland2004scattered,jacot2018neural}.

We first impose the following assumption on the Markov chain
\begin{assumption}
\label{assump:r_pi_P_RKHS}
There exists $R>0$ such that $\max\{\|r\|_\CalH, \max_{s\in\CalS}\|P(s|\cdot,\cdot)\|_\CalH\}\leq R$.    
\end{assumption}
In policy optimization of RL, the RKHS norm  $\|Q^{\pi}\|_{\CalH}$, which is crucial in our analysis, is difficult to determine. The following lemma to show that it is related to the RKHS norm of $\pi$:
\begin{lemma}
\label{lem:Q_pi_RKHS}
    There exists positive constant $C$ that only depends on $R$, such that $\|Q^\pi\|_{\CalH}\leq C\|\pi(\cdot|\cdot)\|_{\CalH}$.
\end{lemma}



\textbf{\textit{Covering Number}.} The following measures of RKHS complexity are fundamental to our analysis:
\begin{definition}[Covering number \& entropy]
    For a given $\delta>0$, the covering number of a RKHS $\CalH$ under  $L_\infty$ norm, denoted by $\CalN(\delta,\|\cdot\|_{L_\infty},\CalH)$, is defined by the smallest integer $M$ sucht that there exists centers $\{f_m\}_{m=1}^M\subseteq\CalH$ for which $\forall f\in\CalH$, $\exists m$: $\|f_m-f\|_{L_\infty}\leq \delta$. The entropy of $\CalH$ is the log of its covering number: $H(\delta,\|\cdot\|_{L_\infty},\CalH)=\log\CalN(\delta,\|\cdot\|_{L_\infty},\CalH)$.
\end{definition}
\begin{assumption}
\label{assump:entropy}
    The entropy of a unit ball in the RKHS $\CalH$: $\CalB=\{f\in\CalH: \|f\|_\CalH\leq 1\}$ satisfies
    \begin{equation}
        H(\delta,\|\cdot\|_{L_\infty},\CalB)\leq C \delta^{-2\beta}|\log \delta|^{2\kappa}
    \end{equation}
    for some $C>0$, $\beta\in[0,1)$, and $\kappa \geq 0$.
\end{assumption}
Assumption~\ref{assump:entropy} covers a broad class of RKHSs, including those previously mentioned, as well as Sobolev spaces with low intrinsic dimension \citep{ding2024random,hamm2021adaptive} and RKHSs induced by deep neural networks \citep{anthony2009neural}.

\section{Policy Evaluation}

For policy evaluation, we generalize the TD learning for parameters in finite-dimensional space to functions in infinitely-dimensional RKHS. We first present the optimization formulation and then our Kernel TD algorithm. Theoretical analysis is provided in the last subsetion.
\subsection{Optimization Formulation}
\label{sec: opt}
In this paper, we study the problem of estimating the Q function by samples from the the Markov chain. We consider, in every update of policy $\pi$, we sample i.i.d. $\{s_0^{(i)}\sim \mu_0\}_{i=1}^n$ from some chosen distribution $\mu_0$ and generate the subsequent state-action pairs $a_0^{(i)}$ and $(s^{(i)}_1,a^{(i)}_1)$ following the Markov chain as follows:
\begin{equation}
    \label{eq:samples_dist}
    s_0^{(i)}\sim \mu,\quad a_0^{(i)}\sim \pi(\cdot|s_0^{(i)}),\quad s_1^{(i)}\sim P(\cdot|s_0^{(i)},a_0^{(i)}),\quad a_1^{(i)}\sim \pi(\cdot|s_1^{(i)}).
\end{equation}
The initial distribution $\mu_0$ can be specified using prior knowledge or obtained by following the chain trajectory and applying MCMC to sample i.i.d. quadruplets ${(s_0,a_0,s_1,a_1)}$. The resulting distribution is only required to satisfy Assumption~\ref{assump:distribution_bounded} together with an additional assumption:

\begin{assumption}
\label{assump:dist_mu_gamma}
    There exists a constant $c$ such that $\max_{(s,a)\in\CalS\times\CalA}P(s'|s,a)\leq c^2 \mu_0(s')$ and $c\gamma<1$.
\end{assumption}
Assumption~\ref{assump:dist_mu_gamma} ensures that the one-step transition is not too far from the initial distribution; otherwise convergence slows. The constant $1-c\gamma$ in the convergence rate reflects sampling complexity, with smaller values indicating higher sample requirements.

We then can have dataset $(\omega_0^{(i)},\omega_1^{(i)})_{i=1}^n$ with $\omega_j^{(i)}=(s_j^{(i)},a_j^{(i)})$, $j=0,1$ for any given policy $\pi$. Corresponding to \eqref{eq:bellman}, we aim to learn the Q-function $Q^\pi$ by solving the following Kernel Ridge Regression (KRR) over the whole RKHS $\CalH$:
\begin{equation}
    \label{eq:KRR_TD}
    \hat{Q}^\pi = \min_{f\in\CalH}\frac{1}{n}\sum_{i=1}^n\left(f(\omega_0^{(i)})-r(\omega_0^{(i)})-\gamma \hat{Q}^\pi(\omega_1^{(i)})\right)^2+\lambda\|f\|_{\CalH}^2.
\end{equation}
The functional minimization problem in \eqref{eq:KRR_TD} is implicit. Nevertheless, by the representer theorem, it can be shown to admit a closed-form solution:

\begin{proposition}
The estimator $\hat{Q}^{\pi}$ has closed-form solution as follows
    \label{prop:KRR_TD_representer_thm}
    \begin{align}
        \label{eq:KRR_TD_closed_form}  \hat{Q}^{\pi}=K(\cdot,\boldsymbol{\omega}_0)\bold{b}^{\pi}
    \end{align}
    where $\Bomega_0=[\omega_0^{(1)},\cdots,\omega_0^{(n)}]^\top\in\Real^n$, $K(\cdot,\Bomega_0)=[K(\cdot,\omega_0^{(1)}),\cdots,[K(\cdot,\omega_0^{(n)})]$, and
    \begin{align*}
        & \bold{b}^{\pi}=\left[\BK+\lambda n {\rm I} -\gamma\bold{C}\right]^{-1}\Br,\quad \BK_{i,j}=K(\omega^{(i)}_0,\omega^{(j)}_0),\\
        & \bold{C}_{i,j}=K(\omega^{(i)}_1,\omega^{(j)}_0),\quad \Br=[r(\omega_0^{(1)}),\cdots,r(\omega_0^{(n)})]^\top.
    \end{align*}
\end{proposition}

 KRR \eqref{eq:KRR_TD}  allows us to extend current standard TD algorithm  for tabular learning to the kernel TD framework.

\subsection{Kernel  Temporal-Difference Learning}
From Proposition~\ref{prop:KRR_TD_representer_thm}, we know that the solution $\hat{Q}^{\pi}\in\hat{\CalH}$, so the KRR \eqref{eq:KRR_TD} can be rewritten as
\begin{equation}
\label{eq:KRR_TD_b}
    \bold{b}^{\pi}= \min_{\bold{b}\in\Real^n}\frac{1}{n}\sum_{i=1}^n\left(K(\omega_0^{(i)},\Bomega_0)\bold{b}-r(\omega_0^{(i)})-\gamma K(\omega_1^{(i)},\Bomega_0)\bold{b}^{\pi}\right)^2+\lambda\bold{b}^\top \BK\bold{b}.
\end{equation}

Instead of directly solving \eqref{eq:KRR_TD_b} for the vector $\Bb^\pi$ by gradient based method on $\Real^n$ , we treat  \eqref{eq:KRR_TD} as a functional optimization on the RKHS $\CalH$. Inspired by Kernel Gradient Descent \citep{ding2024random,lin2018distributed,raskutti2014early}, we propose the following updating rule in the RKHS:
\begin{equation}
    \label{eq:TD_update}
    f_{t+1}=(1-\alpha_t)f_t-\eta_t\sum_{i=1}^n\left(f_t(\omega_0^{(i)})-r(\omega_0^{(i)})-\gamma f_t(\omega_1^{(i)})\right)K(\omega_0^{(i)},\cdot),
\end{equation}
where $\alpha_t$ acts as the weight decay to prevent over-fitting and improve generalization \citep{hu2021regularization}, and $\eta_t$ is the step-size for learning rate. It can be noticed that \eqref{eq:TD_update} can be converted to a form similar to semi-gradient TD(0) \citep{sutton1998reinforcement} if we represented it by $\Bb_t=\BK^{-1}f_t(\Bomega_0)$:
\begin{equation}
    \label{eq:TD_update_b}
    \Bb_{t+1}=(1-\alpha_t)\Bb_t-\eta_t\left(f_{t}(\Bomega_0)-\Br-\gamma f_{t}(\Bomega_1)\right).
\end{equation}
The difference between \eqref{eq:TD_update_b} and semi-gradient TD(0) lies in the fact \eqref{eq:TD_update_b} uses a functional gradient in the infinite-dimensional RKHS, while semi-gradient TD(0) uses a gradient of some parameterized functions in a finite-dimensional space.

A natural question is why we don't directly solve the $n$-dimensional \eqref{eq:KRR_TD_b}. The reason is that the update rule \eqref{eq:TD_update_b}  is more efficient because it uses the RKHS inner product $\langle\cdot,\cdot\rangle_\CalH$ instead of the $l^2$ inner product in $\Real^n$. This change in inner product modifies the gradient and Hessian, effectively serving as a preconditioner that can improve the performance of the algorithm \citep{neuberger2009sobolev}.

From updating rule \eqref{eq:TD_update_b}, we can derivethe convergence pattern of $f_t$ to the target $\hat{Q}^\pi$ for correctly-selected constant weight decay $\alpha_t=\alpha$ and step size $\eta_t=\eta$:
\begin{equation}
    \label{eq:TD_iteration_2}
    \begin{aligned}
    \Bb_{t+1}-\Bb^\pi=&\left[\rm{I}-\left(\alpha\rm{I}+\eta\BK-\eta\gamma\BC\right)\right]\left[\Bb_t-\eta \left(\alpha\rm{I}+\eta\BK-\eta \gamma\BC\right)^{-1}\Br \right]\\
    =& \left[\rm{I}-\left(\alpha\rm{I}+\eta\BK-\eta\gamma\BC\right)\right]\left[\Bb_t-\Bb^\pi\right]\\
    =&\left[((1-\alpha)\rm{I}-\eta\BK+\eta\gamma\BC\right]^{t+1}\left[\Bb_0-\Bb^\pi\right]
\end{aligned}
\end{equation}

If the eigenvalues of $\left[(1-\alpha)\rm{I}-\eta\BK+\eta\gamma\bold{C} \right]$ are small, then $ f_{t}$ converges to $\hat{Q}^{\pi}$ exponentially fast.

\subsection{Convergence Analysis of Kernel TD}
We first show an error decomposition for the estimator \eqref{eq:KRR_TD}. Define the difference function $D^{\pi}=\hat{Q}^{\pi}-Q^{\pi}$ and the Bellman residual:
\begin{equation}
    \label{eq:bellman_residual}
        \varepsilon_i=r(\omega_0^{(i)})+\gamma Q^{\pi}(\omega_1^{(i)})-Q^{\pi}(\omega^{(i)}_{0}).
    \end{equation}
    We have the following error decomposition
    \begin{proposition}[Statistical-Approximation Error Decomposition]
        \label{prop:KRR_TD_error_decomposition}
        \begin{align}
    \label{eq:error_decomposition}
        \frac{1}{n}\sum_{i=1}^n\left(\CalD^{\pi}(\omega^{(i)}_0)^2-\gamma \CalD^{\pi}(\omega^{(i)}_0)\CalD^{\pi}(\omega^{(i)}_1)\right)= {\frac{1}{n}\sum_{i=1}^n\varepsilon_i\CalD^{\pi}(\omega^{(i)}_0)}
        -{\lambda \langle\CalD^{\pi},\hat{Q}^{\pi}\rangle_\CalH}.
    \end{align}
    \end{proposition}
    From the error decomposition on the right-hand side of \eqref{eq:error_decomposition}, we then can use empirical process \citep{geer2000empirical} to derive the following convergence rate of $\hat{Q}^\pi$ on training data

\begin{theorem}
    \label{thm:convergence_Kernel_TD}
    Suppose Assumptions~\ref{assump:distribution_bounded}, \ref{assump:r_pi_P_RKHS}, \ref{assump:entropy}, and \ref{assump:dist_mu_gamma} hold. Let $\eta=C_1/n$, $\alpha=\eta\lambda n$, and iteration number $t\geq C_2\log n\|\Bb_0-\Bb^\pi\|$ for some universal $C_1,C_2>0$ and $\lambda =\CalO((1-c\gamma)^{\frac{\beta}{2+2\beta}}n^{-\frac{1}{2+2\beta}}|\log n|^{\frac{\kappa}{1+\beta}})$, then the kernel TD estimator $f_t$ satisfies
    \begin{equation}
    \label{eq:convergence_Kernel_TD}
        \begin{aligned}
             &\sqrt{\frac{1}{n}\sum_{i=1}^n\left|f_t(\omega_0^{(i)})-Q^\pi(\omega_0^{(i)})\right|^2}\leq \CalO_p\left((1-c\gamma)^{-\frac{2+\beta}{2+2\beta}}n^{-\frac{1}{2+2\beta}}|\log n|^{\frac{\kappa}{1+\beta}}\right)\|Q^\pi\|_{\CalH},\\
    &\|f_t\|_{\CalH}\leq \CalO_p(1)\|Q^\pi\|_{\CalH}.
        \end{aligned}
        \end{equation}
\end{theorem}

Theorems~\ref{thm:convergence_Kernel_TD} yields convergence only on the dataset $\{\omega_0^{(i)}\}_{i=1}^n$, not generalization error, since Assumption~\ref{assump:entropy} is imposed solely for general applicability. Establishing generalization requires additional structural information on the RKHS, after which it follows readily from sampling inequalities.
\begin{corollary}
\label{coro:convergence_L2}
    Suppose the same assumptions in Theorem~\ref{thm:convergence_Kernel_TD} hold, then
    \begin{enumerate}
        \item[\textbf{Tabular:}]: If $\CalS\times\CalA$ are uncorrelated discrete set $\{\omega_j:j=1,\cdots, n^\nu\}$ for some constant $\nu\in(0,1)$, then $K=\delta_{s=s'}\delta_{a=a'}$ and the Kernel TD estimator satisfies
        \begin{equation}
            \label{eq:convergence_tabular}
            \|f_t-Q^\pi\|_{L_2(\sigma^\pi_0)}\leq \CalO_p\left((1-c\gamma)^{-1}n^{-\frac{1}{2}}|\log n|^{1/2}+\frac{1}{n^{(1+\nu)/4}}\right)\|Q^\pi\|_{\CalH};
        \end{equation}
        \item[\textbf{Sobolev:}] If $K=K_\CalS\delta_{a=a'}$ where $K_\CalS$ is a Sobolev type kernel with smoothness  $m$ and intrinsic dimension $d$ and $\CalA=\{a\}_{a=1}^A$, then the Kernel TD estimator  satisfies
\begin{equation}
            \label{eq:convergence_Sobolev}
            \|f_t-Q^\pi\|_{L_2(\sigma^\pi_0)}\leq \CalO_p\left((1-c\gamma)^{-\frac{2m+d/2}{2m+d}}n^{-\frac{m}{2m+d}}\right)\|Q^\pi\|_{\CalH}.
        \end{equation}
        \item[\textbf{NTK:}] If $K=N\delta_{a=a'}$ where $N$ is the NTK $N(s,s')$ of a two-layer neural network on a $d$-sphere $\mathbb{S}^{d-1}$  and $\CalA=\{a\}_{a=1}^A$, then the Kernel TD estimator satisfies
        \begin{equation}
            \label{eq:convergence_NTK}
            \|f_t-Q^\pi\|_{L_2(\sigma^\pi_0)}\leq \CalO_p\left((1-c\gamma)^{-\frac{3d+1}{4d}}n^{-\frac{d+1}{4d}}\right)\|Q^\pi\|_{\CalH}.
        \end{equation}
        \item[\textbf{Gaussian:}]  If $\CalH$ is Gaussian RKHS, i.e., $K=e^{-\|\omega-\omega'\|^2}$ with  $\CalS\times\CalA=[0,1]^{d}$  hypercube, then 
        \begin{equation}
        \label{eq:convergence_Gaussian}
           \|f_t-Q^\pi\|_{L_2(\sigma^\pi_0)}\leq  \CalO_p\left((1-c\gamma)^{-1}n^{-\frac{1}{2}}|\log n|^{{(d+1)}/{2}}\right)\|Q^\pi\|_{\CalH}.
        \end{equation}

    \end{enumerate}
\end{corollary}

\section{Policy Improvement}

We generalize the popular NPG update rule to a version for policy evaluation in the infinite-dimensional RKHS. In this section, we first present our algorithm for updating the policy function from a functional perspective, then we provide global convergence for our algorithm. 

\subsection{Natural Policy Gradient in RKHS}
\label{sec: NPG}

We consisder updating policy  by NPG:
\begin{equation}
    \label{eq:NPG}
    \pi^{k+1}\propto \pi^k \exp\{\Delta_k f_T^{(k)}\}\propto\exp\{\sum_{j=0}^k\Delta_j f_T^{(j)}\}
\end{equation}
where $f_T^{(j)}$ is the kernel TD estimator of $Q^{(j)}:=Q^{\pi^j}$, trained for $T$ steps on initial data $\{s_0^{(i)}\}_{i=1}^n$ and $\Delta_j$ is the step size. \eqref{eq:NPG} is the solution of the following functional proximal gradient optimization:
\begin{lemma}
    \label{lem:NPG_optimize}
     Suppose $\pi^k\propto \exp\{F\}$ and $1>\pi^{(k+1)}>0$. Then on the sets $\{s^{(i)}_0\}\times\CalA$,
    \begin{equation}
\label{eq:NPG_optimize}
    \pi^{k+1}=\argmax_{\int_\CalA\pi(a|\cdot)da=1,\pi\geq 0} \E_n\left[\Delta_k\int_\CalA f^{(k)}(s,a)\pi(a|s)da- \text{KL}\left(\pi(\cdot|s)\|\pi^k(\cdot|s)\right)\right]
\end{equation}
where $\E_n[f(s)]=\frac{1}{n}\sum_{i=1}^nf(s_0^{(i)})$ is the empirical expectation induced by dataset $\{s_0^{(i)}\}$ and KL$(p\|q)$ denote the KL divergence between distributions $p$ and $q$.
\end{lemma}

 The update rule \eqref{eq:NPG} generalizes existing NPG methods, including those for tabular RL \citep{liu2024elementary}, linear RL \citep{cen2022fast}, and two-layer neural networks \citep{liu2019neural}, to general RKHS functions. Our update rule \eqref{eq:NPG_optimize} allows the policy $\pi^{k}$ to be a continuous function of $(s,a)$ and replaces the population-level expectation with an empirical expectation that is adaptive to the data.

By combining the NPG update rule \eqref{eq:NPG} with the kernel TD update rule \eqref{eq:TD_update}, we introduce an efficient algorithm that generalizes NPG to infinite-dimensional RKHS as follows:

\begin{algorithm}
	\caption{NPG in RKHS} 
    \label{alg:PPO_RKHS}
	\begin{algorithmic}[1]
     \State \textbf{Require}: MDP $(\CalS,\CalA,P,r,\gamma)$, RKHS $\CalH$ 
     \State Initialize $\pi^0\propto \exp\{f^{(0)}\}$ for some $f^{(0)}\in\hat{\CalH}$, set $F=\Delta_0f^{(0)}$
		\For {$k=1,,2,\ldots$}
        \State Select an initial sampling distribution $\mu_0^k$ and number of samples $n\leftarrow n^{(k)}$
        \State Generate $\{[s_0^{(i)},a_0^{(i)},s_1^{(i)},a_1^{(i)}]\sim \mu^k_0(s_0)\pi^{k-1}(a_0|s_0)P(s_1|s_0,a_0)\pi^{k-1}(a_1|s_1)\}_{i=1}^n$
        \State Set $T\leftarrow T^{(k)}$, $\alpha\leftarrow\alpha^{(k)}$, $\eta\leftarrow\eta^{(k)}$ 
        \State Initialize $f^{(k)}$
			\For {$t=1,\cdots, T$}
                \State Update $f^{(k)}\leftarrow(1-\alpha)f^{(k)}-\eta\sum_{i=1}^n\left(f^{(k)}(\omega_0^{(i)})-r(\omega_0^{(i)})-\gamma f^{(k)}(\omega_1^{(i)})\right)K(\omega_0^{(i)},\cdot)$
                
			\EndFor
			\State $F \leftarrow F+\Delta_k f^{(k)} $, $\pi^k\propto \exp\{F\}$
		\EndFor
	\end{algorithmic} 
\end{algorithm}

From the NTK perspective, Algorithm~\ref{alg:PPO_RKHS} can be viewed as the evolution of a deep neural network described by kernel $K$. Consequently, it remains valid if we replace the RKHS functions $f^{(k)}$
  with a deep neural net $f_{\Btheta^{(k)}}$ parameterized by $\Btheta^{(k)}$ and the policy as $\text{Softmax}(f_{\Btheta^{(k)}})$.

\subsection{Global Convergence of NPG}

We first define the expected total reward to measure the optimality of a policy $\pi$
\begin{equation}
\label{eq:performance_difference}
    \CalR[\pi]=\E_{S\sim \nu^*}[V^{\pi}(S)]=\E_{S\sim \nu^*}[\langle Q^{\pi}(S,\cdot),\pi(\cdot|S)\rangle_\CalA].
\end{equation}

We can have a fundamental inequality for the \eqref{eq:performance_difference} in RKHS. This inequality, which slightly modifies the mirror descent analysis in \cite{nesterov2013introductory,Nemirovsky1985} for infinite-dimensional spaces, is central to our analysis.
\begin{theorem}
\label{thm:one_step_improvement}
  In Algorithm~\ref{alg:PPO_RKHS}, 
  \begin{equation}
      \label{eq:one_step_improvement}
     \begin{aligned}
      &\inf_{k}\left(\CalR[\pi^k]-\CalR[\pi^*]\right)\\
      \leq& \frac{\left(\sum_{k}2\Delta_k\|f^{(k)}-Q^{(k)}\|_{L_\infty}\right)+\left(\sum_k\Delta^2_k(1-\gamma)^{-1}\|r\|_{L_\infty}\right)+\E_{S\sim\nu^*}\text{KL}\left(\pi^*(\cdot|S)||\pi^{0}(\cdot|S)\right)}{\sum_k\Delta_k}.
     \end{aligned}
  \end{equation}
\end{theorem}

From Theorem~\ref{thm:one_step_improvement}, we observe that achieving the optimal $\CalO(k^{-1/2})$ convergence rate for stochastic optimization requires selecting suitable parameters in Algorithm~\ref{alg:PPO_RKHS} to ensure the policy evaluation error is well-controlled under the $L_\infty$ norm. This leads to the following corollary.
\begin{corollary}
\label{coro:NGP_convergence}
Set $\Delta_k=1/\sqrt{k}$. For settings listed in Corollary~\ref{coro:convergence_L2}, set $n^{(k)}$ and $\lambda^{(k)}$ according to Table~\ref{tab:NPG_parameter}. Set $\alpha^{(k)},\eta^{(k)}$, and $T^{(k)}$ according to Theorem~\ref{thm:convergence_Kernel_TD} with $\lambda=\lambda^{(k)}$. Then under the same conditions as Corollary~\ref{coro:convergence_L2},  we have
\begin{equation}
    \label{eq:NGP_convergence}
    \inf_{1\leq k\leq k^*}\left(\CalR[\pi^k]-\CalR[\pi^*]\right)\leq \CalO_p(\frac{1}{\sqrt{k^*}}).
\end{equation}



\begin{table}[H]
    \centering
\caption{Parameters selection in Algorithm~\ref{alg:PPO_RKHS} for four cases in Corollary~\ref{coro:convergence_L2}}
\label{tab:NPG_parameter}.
\begin{tabular}{ |p{1.2cm}|p{5cm}|p{5cm}|  }
 \hline Setting
 & $n^{(k)}$ & $\lambda^{(k)}$\\
 \hline
 Tabular   & $\CalO\left(\begin{aligned} & \frac{\|\pi^k\|^2_\CalH k}{(1-c\gamma)^2}\log \frac{\|\pi^k\|_\CalH k}{1-c\gamma} \\ & + (\sqrt{k}\|\pi^k\|_\CalH)^{\frac{4}{1+\nu}} \end{aligned}\right)$      & $\CalO\left(\frac{1-c\gamma}{\|\pi^k\|_\CalH\sqrt{k}}\right)$\\
  \hline
  Sobolev & $\CalO\bigg(\frac{\|\pi^k\|_\CalH^{\frac{2(2m+d)}{2m-d}}k^{\frac{2m+d}{2m-d}}}{(1-c\gamma)^{\frac{2m+d/2}{m}}}\bigg)$ & $\CalO\left(\frac{1-c\gamma}{\|\pi^k\|_{\CalH}^{\frac{2m}{2m-d}}k^{\frac{m}{2m-d}}}\right)$
\\
\hline
NTK  & $\CalO\left(\frac{\|\pi^k\|_\CalH^{2d}k^{d}}{(1-c\gamma)^{\frac{3d+1}{d+1}}}\right)$ & $\CalO\left(\frac{1-c\gamma}{\|\pi^k\|_{\CalH}^{\frac{d+1}{2}}k^{\frac{d+1}{4}}}\right)$ \\
\hline
Gaussian &$\CalO\left(\frac{\|\pi^k\|^{\frac{2}{1-\epsilon}}k^{\frac{1}{1-\epsilon}}}{(1-c\gamma)^2}\log \frac{\|\pi^k\|_\CalH k}{1-c\gamma}\right)$  & $\CalO\left(\frac{(1-c\gamma)}{\|\pi\|_{\CalH}^{\frac{1}{1-\epsilon}}\sqrt{k}^{\frac{1}{1-\epsilon}}}\right),\ \forall\epsilon\in(0,1)$\\
\hline
\end{tabular}

\end{table}
    
\end{corollary}

Corollary~\ref{coro:NGP_convergence} shows that the required sampling number increases with both the step size $k$ and the policy complexity, measured by the RKHS norm $\|\pi^k\|_\CalH$. As the policy nears the optimum, the performance gap narrows, requiring more accurate evaluation and hence more samples. As NPG progresses, the policy may also become highly concentrated, approaching a delta distribution; in this case, the RKHS norm diverges and the sampling requirement grows accordingly.

The RKHS norm is a well-defined and meaningful measure of policy complexity. For examples, in the Sobolev and NTK cases, the policy is from composing a softmax with RKHS function, and remains in the RKHS since the softmax is analytic. Although the RKHS norm is hard to estimate in the NTK case—where the policy corresponds to a deep neural network—other attributes such as network stability and architecture can serve as practical proxies for the RKHS norm complexity.

\section{Experiments}
We conduct numerical experiments to empirically validate our theoretical analysis of Q-function estimation within the NPG framework outlined in Algorithm~\ref{alg:PPO_RKHS}. The experiments are designed to investigate the convergence properties of a deep neural network-based implementation and sample efficiency.

\begin{wrapfigure}[24]{r}{0.6\textwidth} 
\centering
\captionsetup[subfigure]{justification=centering}

\begin{subfigure}{0.48\linewidth}
  \centering
  \includegraphics[width=\linewidth]{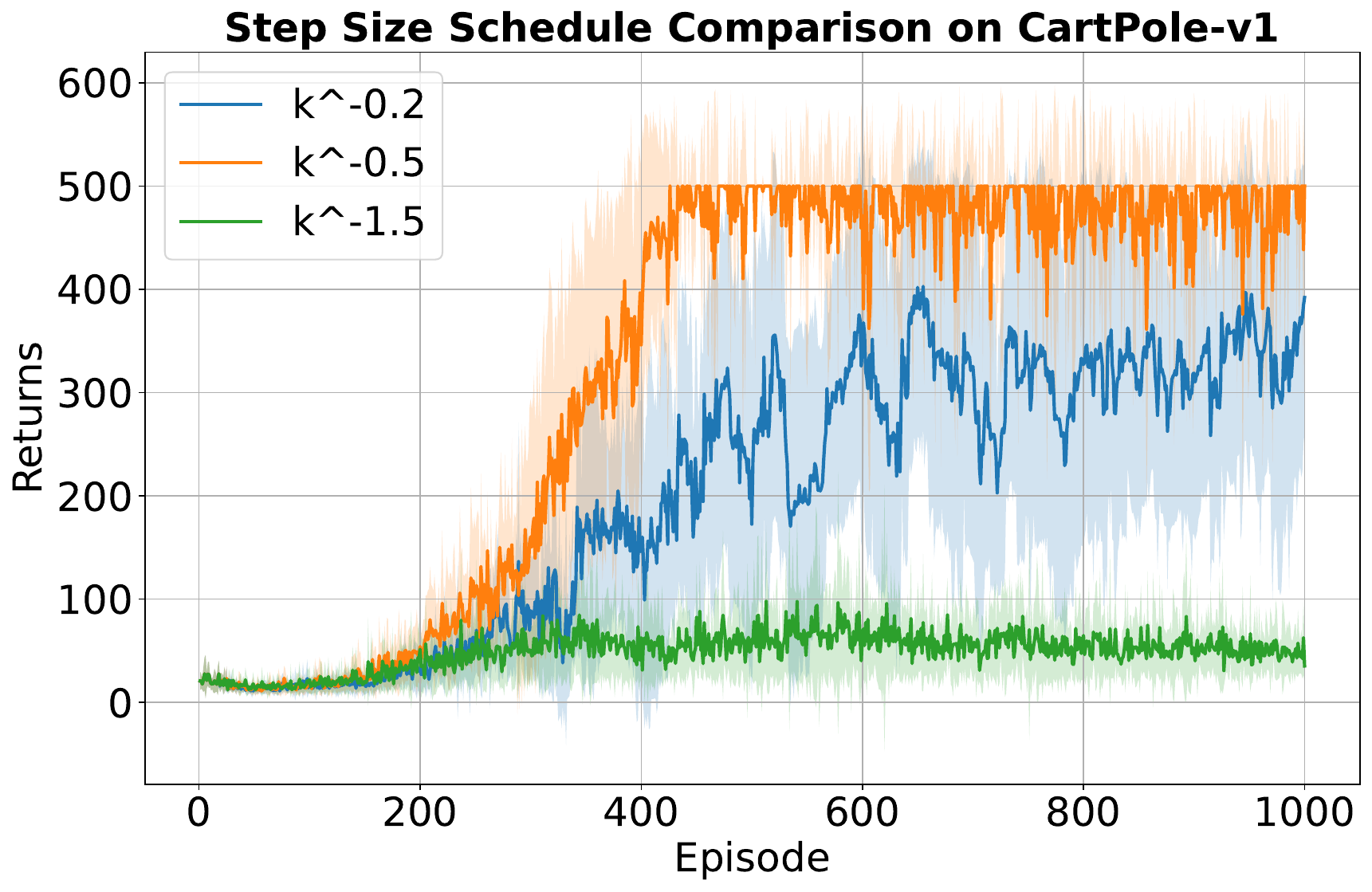}
  \caption{Reward (CartPole-v1).}
\end{subfigure}\hfill
\begin{subfigure}{0.48\linewidth}
  \centering
  \includegraphics[width=\linewidth]{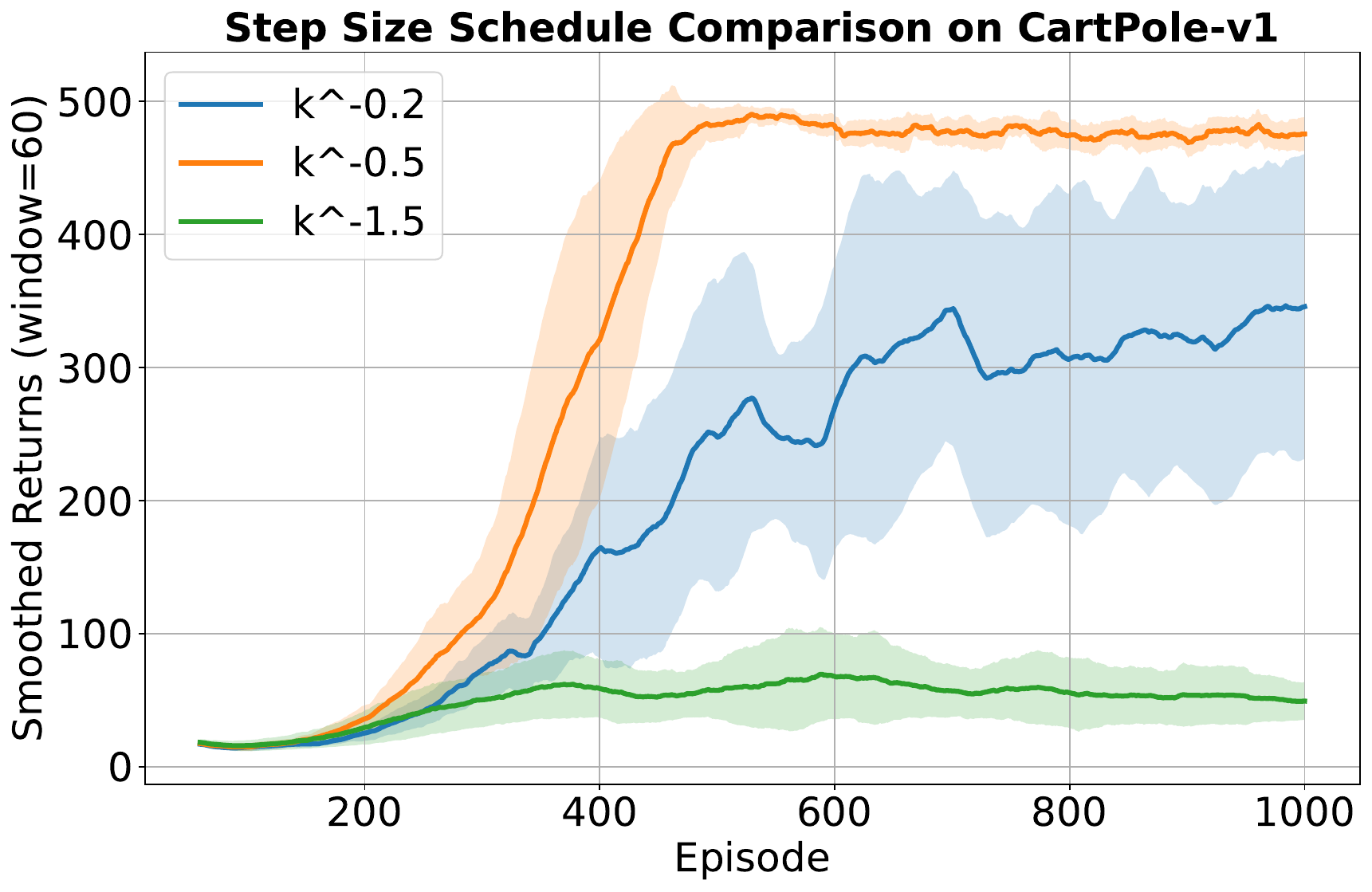}
  \caption{Smoothed (CartPole-v1).}
\end{subfigure}

\vspace{3pt}

\begin{subfigure}{0.48\linewidth}
  \centering
  \includegraphics[width=\linewidth]{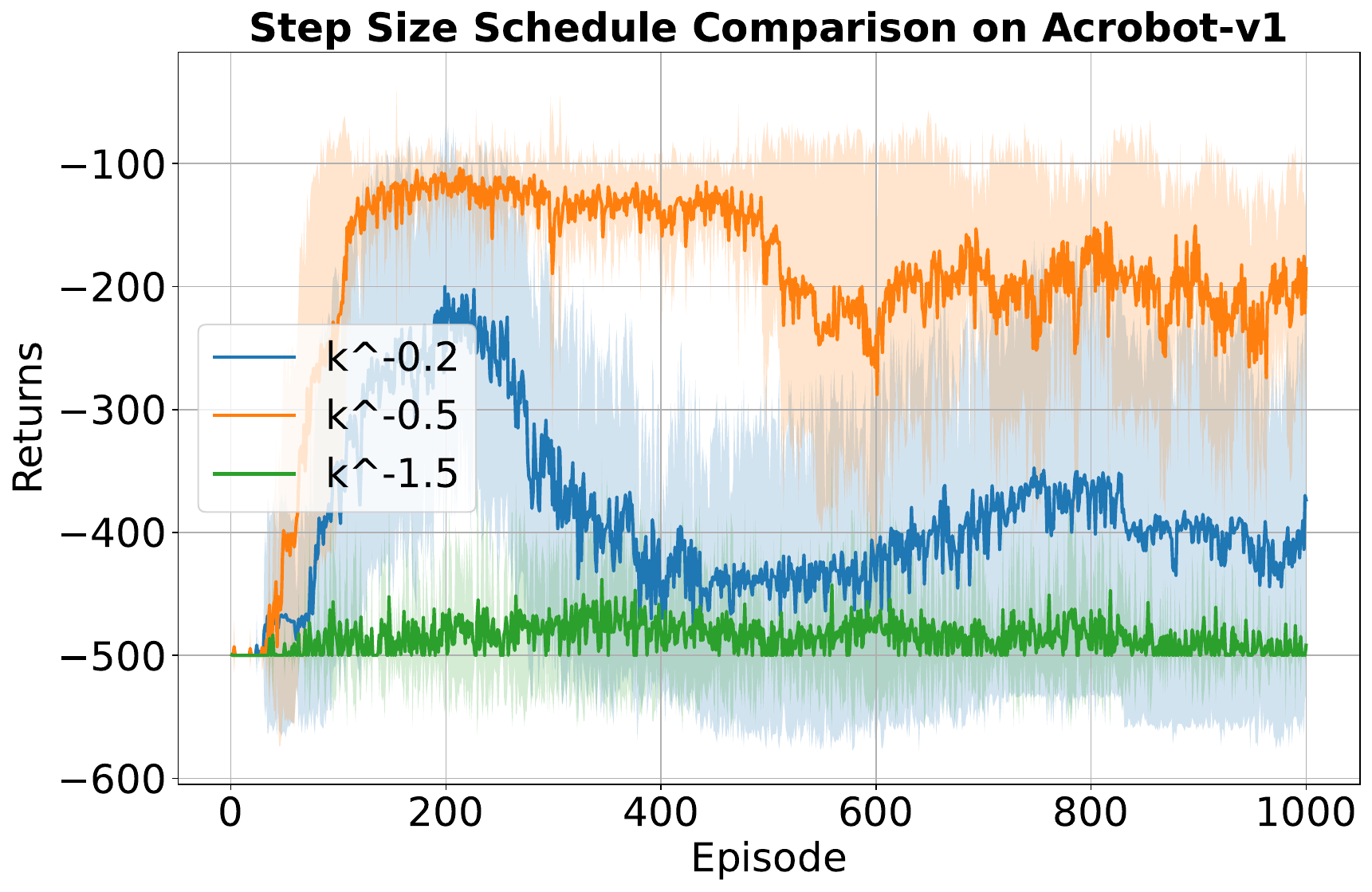}
  \caption{Reward (Acrobot-v1).}
\end{subfigure}\hfill
\begin{subfigure}{0.48\linewidth}
  \centering
  \includegraphics[width=\linewidth]{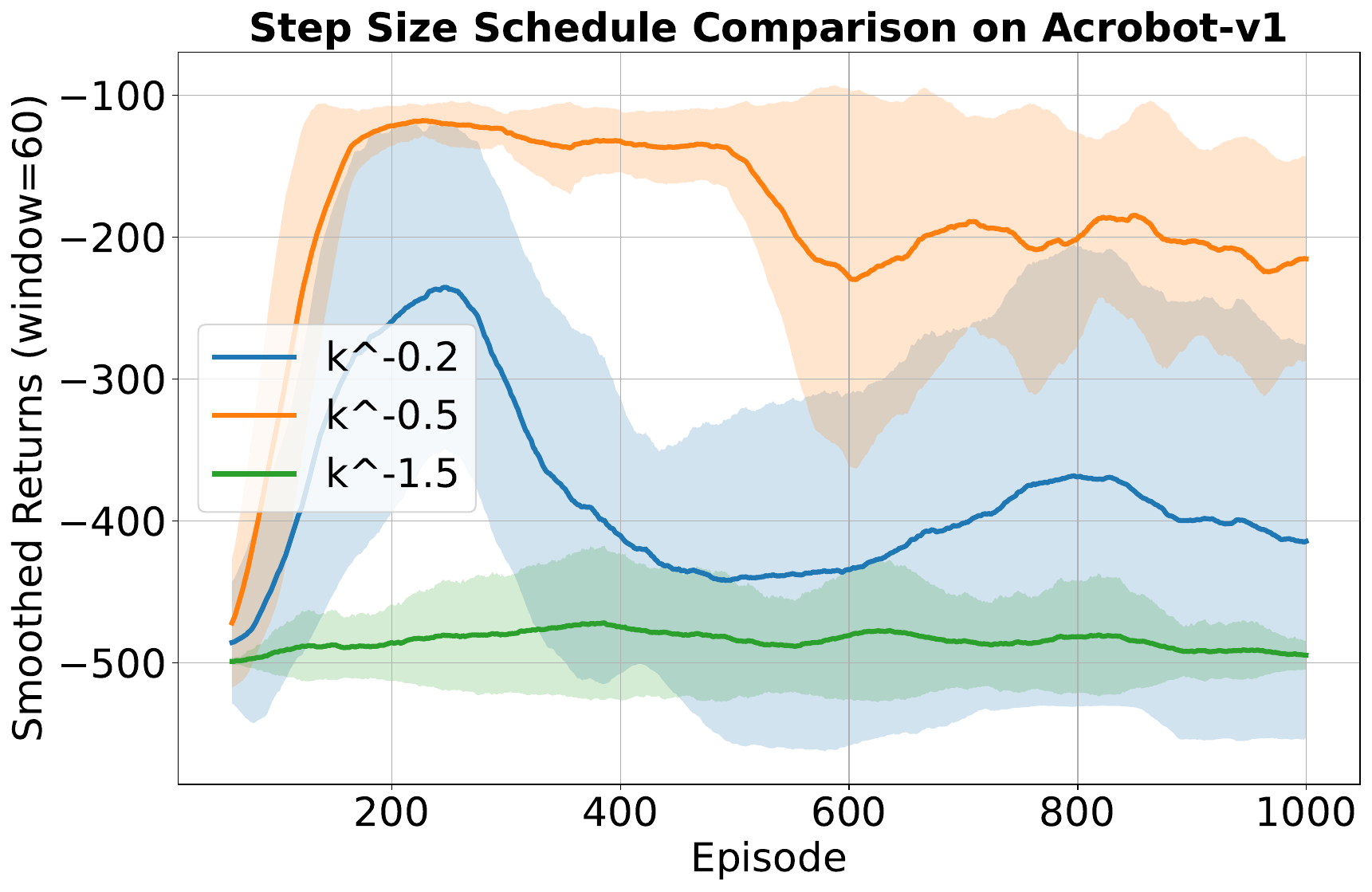}
  \caption{Smoothed (Acrobot-v1).}
\end{subfigure}

\caption{Convergence analysis on CartPole-v1 and Acrobot-v1. We show both raw and smoothed rewards (using a window size of 60). Both environments exhibit a consistent convergence pattern under different step sizes: $\Delta_k = k^{-0.5}$ robustly converges to optimal performance, whereas $\Delta_k = k^{-0.2}$ diverges and $\Delta_k = k^{-1.5}$ stagnates.}
\label{fig:convergence-cartpole-acrobot}
\end{wrapfigure}

\subsection{Experimental Setup}

\textbf{Environments.}  
We evaluate our method on standard benchmarks: \textit{CartPole-v1} \citep{barto2012neuronlike}, a classic low-dimensional control task (4-dim continuous states and 2-dim discrete actions), and \textit{Acrobot-v1} \citep{sutton1998reinforcement}, a more challenging underactuated robotics problem (6-dim continuous states and 3-dim discrete actions. The details of these two environments are in the Appendix~\ref{app:env}.

\textbf{Model Architecture.} We instantiate the non-parametric function $f^{(k)}$ from Algorithm 1 with a neural network $f_{\theta^{(k)}}$, where the policy is given by $\pi_\theta \propto \text{Softmax}(f_{\theta})$. The network is a two-layer MLP.
This model is implemented within an Actor-Critic paradigm where the MLP backbone serves as the Critic ($f_\theta$), outputting Q-values, and the addition of a Softmax layer forms the Actor ($\pi_\theta$). The Actor and Critic share all parameters, enabling efficient end-to-end training.

\textbf{Algorithm Configuration.} Algorithm~\ref{alg:PPO_RKHS} outlines a two-stage process: optimizing the Q-function for $T$ epochs, followed by a single policy update. 
In practice, a finite number of updates (e.g., $T=4$) is insufficient for the Critic to learn well. Consequently, updating the Actor based on an imperfect Critic leads to inefficient and unstable learning. To address this, we adopt a joint optimization strategy, a standard practice in algorithms like PPO. We combine the Critic's TD Error loss (Line 9) and the Actor's NPG objective (Line 12) into a unified loss function. This objective is then minimized over $T$ epochs for each batch of collected data. 

Each outer iteration $k$ proceeds as follows: we execute the current policy to collect $n$ samples, after which we run $T$ optimization epochs over this batch to update the shared actor–critic parameters. Experiments are run for 1000 episodes and averaged over 10 random seeds; we report the mean and standard deviation across seeds. Hyperparameters are shown in Table \ref{tab:hyper_1} in Appendix \ref{app:hyper}.

\subsection{Convergence Analysis}
\label{sec:converge}
We investigate the empirical impact of the step size schedule, $\Delta_k = k^{-\alpha}$, by testing three theoretically motivated exponents: $\alpha \in \{0.2, 0.5, 1.5\}$.

\textbf{Results Analysis.}
Figure~\ref{fig:convergence-cartpole-acrobot} corroborates our theoretical predictions, revealing three distinct learning dynamics determined by the schedule $\alpha$.
These results empirically confirm that the step size schedule is a critical factor governing the stability and efficiency of NPG with TD learning. The narrow confidence intervals across 10 seeds highlight the statistical reliability of our findings, providing strong evidence for the optimal convergence for the $\Delta_k = \frac{1}{\sqrt{k}}$ schedule in Corollary \ref{coro:NGP_convergence}. Moreover, as shown in Figure~\ref{fig:convergence-cartpole-acrobot} (c-d), the learning curve for $\Delta_k = \tfrac{1}{\sqrt{k}}$ exhibits a decline after some episodes. This occurs because, as the neural network becomes more complex, the available samples are insufficient to fully support training, thereby weakening the effectiveness of optimization. A similar trend emerges when the CartPole experiments run for longer episodes (e.g., Figure \ref{fig:comparison} (a-b)).

\subsection{Running Efficiency Analysis}
\label{sec:ppo_comp}
To analyze the sample efficiency, we compare our Algorithm with Standard PPO\citep{schulman2017proximal}. The key architectural difference between these methods lies in their data requirements for value estimation: Standard PPO utilizes a state-value critic ($V(s)$) and relies on Generalized Advantage Estimation (GAE). GAE requires complete trajectory $(s_0, a_0, r_0, s_1, a_1, ..., s_t, a_t)$ segments to recursively compute advantage targets, making its update rule non-local in time. Instead, Our NPG employs an action-value critic ($Q(s,a)$) and learns from single-step Temporal Difference (TD) errors. This approach is highly data-efficient, as it only requires $(s_0, a_0, r_0, s_1, a_1)$ learning Q.

\begin{figure}[h]
\centering
\begin{subfigure}{0.31\textwidth}
    \includegraphics[width=\textwidth]{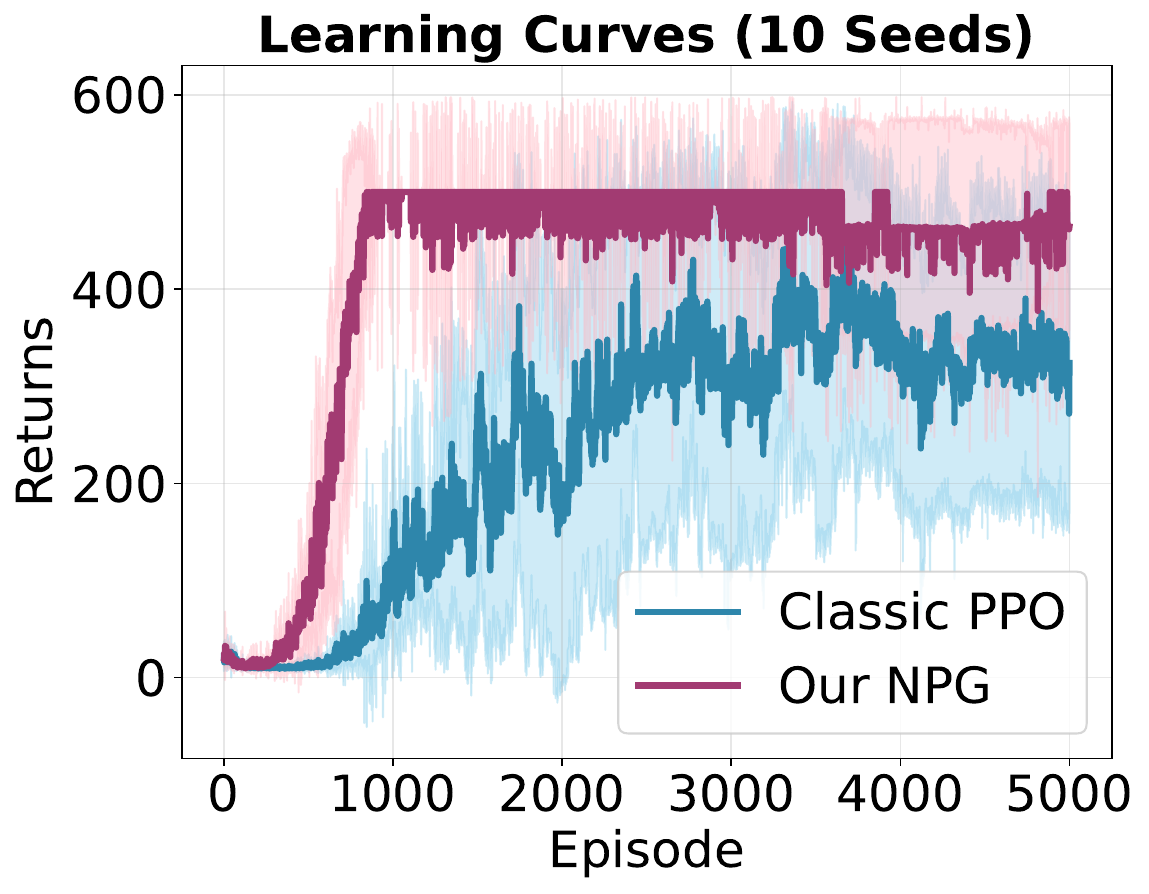}
    \caption{Reward.}
    \label{fig:sub2}
\end{subfigure}
\hfill
\begin{subfigure}{0.31\textwidth}
    \includegraphics[width=\textwidth]{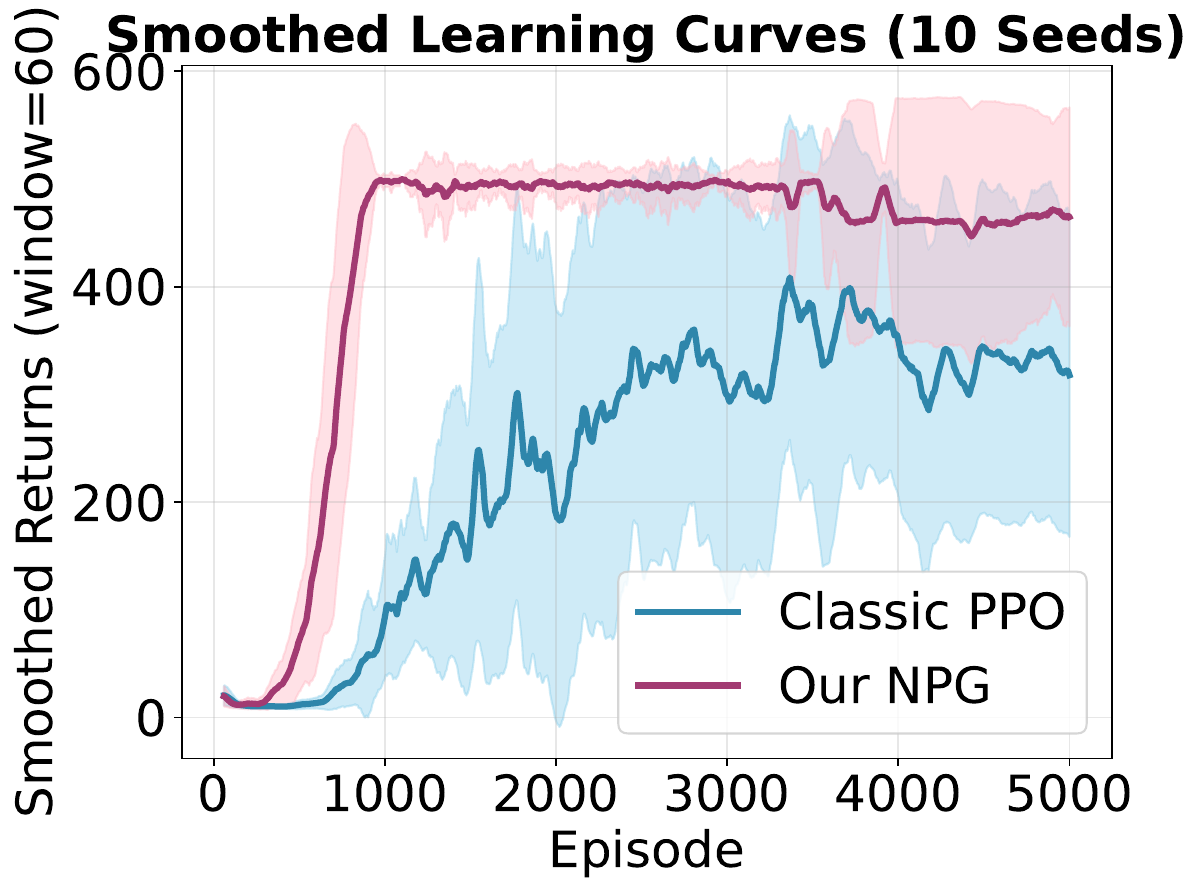}
    \caption{Smoothed reward.}
    \label{fig:sub1}
\end{subfigure}
\hfill
\begin{subfigure}{0.31\textwidth}
    \includegraphics[width=\textwidth]{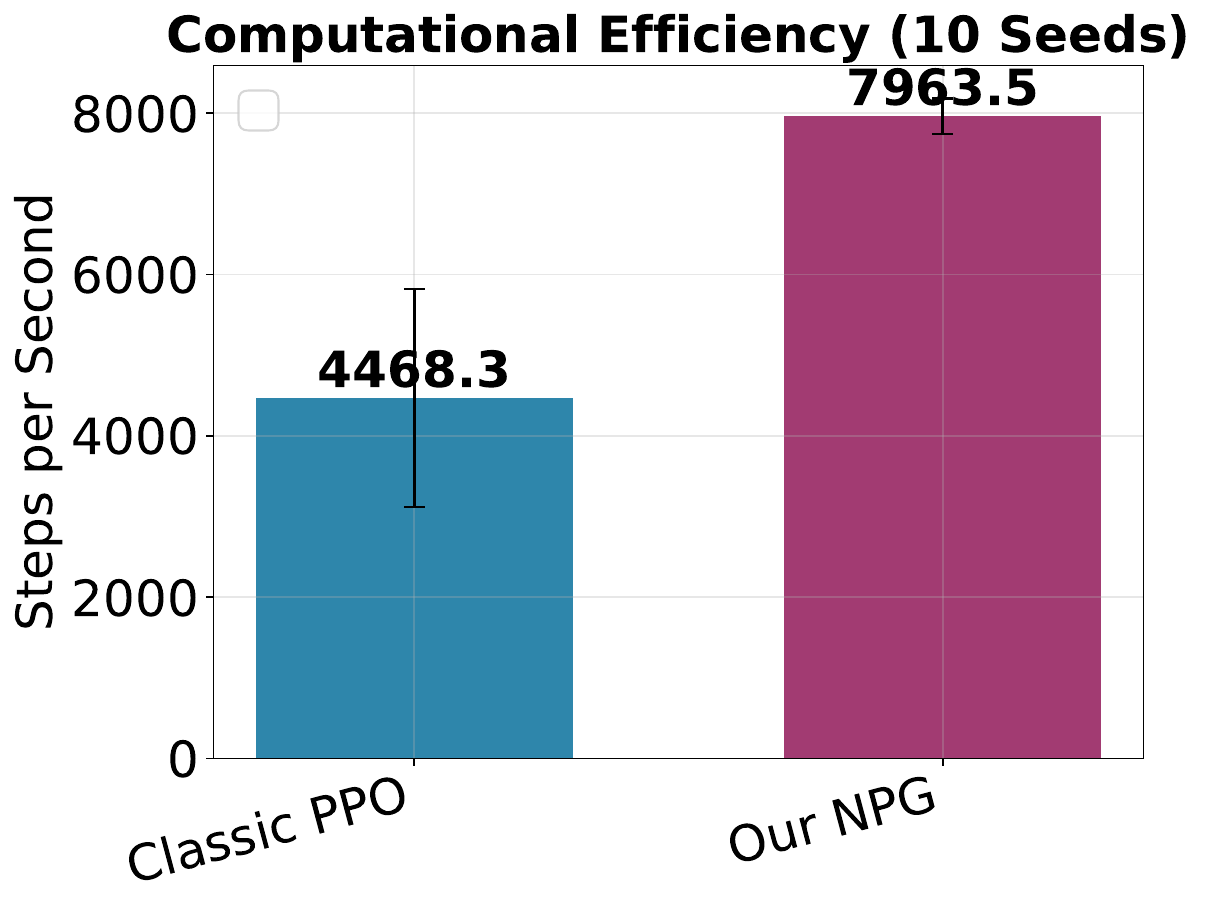}
    \caption{Processed steps/second.}
    \label{fig:sub3}
\end{subfigure}
\vspace{-3pt}
\caption{Performance and efficiency comparison on CartPole-v1 over 10 seeds.} 
\label{fig:comparison}
\end{figure}

\textbf{Results Analysis.}
Our experiments demonstrate that our NPG with TD learning significantly outperforms Standard PPO in both sample efficiency and computational efficiency. As shown in Figure~\ref{fig:comparison} (a–b), our NPG consistently solves the task, achieving the maximum reward of 500. In contrast, Standard PPO converges more slowly and exhibits high variance, with only some seeds reaching the optimal solution. Since our algorithm depends on sampling quality (Assumption~\ref{assump:distribution_bounded} and~\ref{assump:dist_mu_gamma}), choosing a suitable distribution enables fast convergence, as confirmed in Figure~\ref{fig:comparison}. The detailed Hyperparameter setting are explained in the Appendix~\ref{app:hyper}.

Figure \ref{fig:comparison}(c) shows our NPG is about 78.2\% more efficient, processing 7,963 state-actions pairs (i.e. $\omega_t^{(i)}, t = 0,1$) per second compared to PPO's 4,468 pairs (i.e. $\omega_t^{(i)}$) per second. 
This computational advantage stems directly from our algorithm's design. The update for our NPG is based on a computationally lightweight, one-step TD-error calculation. In contrast, Standard PPO must perform a more complex, recursive GAE calculation over entire trajectory segments for each update.


\section{Conclusions}

We generalize TD and NPG to infinite-dimensional RKHSs and develop a unified framework for analyzing the convergence rate of kernelized TD in terms of sample size, as well as the sample complexity of RKHS-based NPG for attaining the $\CalO_p(k^{-1/2})$ rate. The framework is illustrated using four commonly studied RKHSs, and numerical experiments further confirm that the theoretical results align well with empirical performance.

\newpage

\bibliography{iclr2026_conference}

\begin{thebibliography}{63}
\providecommand{\natexlab}[1]{#1}
\providecommand{\url}[1]{\texttt{#1}}
\expandafter\ifx\csname urlstyle\endcsname\relax
  \providecommand{\doi}[1]{doi: #1}\else
  \providecommand{\doi}{doi: \begingroup \urlstyle{rm}\Url}\fi

\bibitem[Adams \& Fournier(2003)Adams and Fournier]{adams2003sobolev}
Robert~A Adams and John J~F Fournier.
\newblock \emph{Sobolev Spaces}.
\newblock Academic Press, 2003.

\bibitem[Agarwal et~al.(2020)Agarwal, Kakade, Lee, and Mahajan]{agarwal2020optimality}
Alekh Agarwal, Sham~M Kakade, Jason~D Lee, and Gaurav Mahajan.
\newblock Optimality and approximation with policy gradient methods in markov decision processes.
\newblock In \emph{Conference on learning theory}, pp.\  64--66. PMLR, 2020.

\bibitem[Agarwal et~al.(2021)Agarwal, Kakade, Lee, and Mahajan]{agarwal2021theory}
Alekh Agarwal, Sham~M Kakade, Jason~D Lee, and Gaurav Mahajan.
\newblock On the theory of policy gradient methods: Optimality, approximation, and distribution shift.
\newblock \emph{Journal of Machine Learning Research}, 22\penalty0 (98):\penalty0 1--76, 2021.

\bibitem[Agazzi \& Lu(2019)Agazzi and Lu]{agazzi2019temporal}
Andrea Agazzi and Jianfeng Lu.
\newblock Temporal-difference learning for nonlinear value function approximation in the lazy training regime.
\newblock 2019.

\bibitem[Anthony \& Bartlett(2009)Anthony and Bartlett]{anthony2009neural}
Martin Anthony and Peter~L Bartlett.
\newblock \emph{Neural network learning: Theoretical foundations}.
\newblock cambridge university press, 2009.

\bibitem[Bagnell \& Schneider(2003)Bagnell and Schneider]{bagnell2003policy}
J~Andrew Bagnell and Jeff Schneider.
\newblock Policy search in kernel hilbert space.
\newblock 2003.

\bibitem[Barto et~al.(2012)Barto, Sutton, and Anderson]{barto2012neuronlike}
Andrew~G Barto, Richard~S Sutton, and Charles~W Anderson.
\newblock Neuronlike adaptive elements that can solve difficult learning control problems.
\newblock \emph{IEEE transactions on systems, man, and cybernetics}, \penalty0 (5):\penalty0 834--846, 2012.

\bibitem[Baxter \& Bartlett(2000)Baxter and Bartlett]{baxter2000direct}
Jonathan Baxter and Peter~L Bartlett.
\newblock Direct gradient-based reinforcement learning.
\newblock In \emph{2000 IEEE International Symposium on Circuits and Systems (ISCAS)}, volume~3, pp.\  271--274. IEEE, 2000.

\bibitem[Bethke et~al.(2008)Bethke, How, and Ozdaglar]{bethke2008kernel}
Brett Bethke, Jonathan~P How, and Asuman Ozdaglar.
\newblock Kernel-based reinforcement learning using bellman residual elimination.
\newblock \emph{Journal of Machine Learning Research (to appear)}, 2008.

\bibitem[Bhandari \& Russo(2024)Bhandari and Russo]{bhandari2024global}
Jalaj Bhandari and Daniel Russo.
\newblock Global optimality guarantees for policy gradient methods.
\newblock \emph{Operations Research}, 72\penalty0 (5):\penalty0 1906--1927, 2024.

\bibitem[Bhandari et~al.(2018)Bhandari, Russo, and Singal]{bhandari2018finite}
Jalaj Bhandari, Daniel Russo, and Raghav Singal.
\newblock A finite time analysis of temporal difference learning with linear function approximation.
\newblock In \emph{Conference on learning theory}, pp.\  1691--1692. PMLR, 2018.

\bibitem[Birman \& Solomjak(1967)Birman and Solomjak]{birman1967piecewise}
M~{\v{S}} Birman and Mikhail~Z Solomjak.
\newblock Piecewise-polynomial approximations of functions of the classes.
\newblock \emph{Mathematics of the USSR-Sbornik}, 2\penalty0 (3):\penalty0 295, 1967.

\bibitem[Brandfonbrener \& Bruna(2019)Brandfonbrener and Bruna]{brandfonbrener2019geometric}
David Brandfonbrener and Joan Bruna.
\newblock Geometric insights into the convergence of nonlinear td learning.
\newblock \emph{arXiv preprint arXiv:1905.12185}, 2019.

\bibitem[Cai et~al.(2019)Cai, Yang, Lee, and Wang]{cai2019neural}
Qi~Cai, Zhuoran Yang, Jason~D Lee, and Zhaoran Wang.
\newblock Neural temporal-difference learning converges to global optima.
\newblock \emph{Advances in Neural Information Processing Systems}, 32, 2019.

\bibitem[Cai et~al.(2020)Cai, Yang, Jin, and Wang]{cai2020provably}
Qi~Cai, Zhuoran Yang, Chi Jin, and Zhaoran Wang.
\newblock Provably efficient exploration in policy optimization.
\newblock In \emph{International Conference on Machine Learning}, pp.\  1283--1294. PMLR, 2020.

\bibitem[Cayci et~al.(2024)Cayci, He, and Srikant]{cayci2024convergence}
Semih Cayci, Niao He, and Rayadurgam Srikant.
\newblock Convergence of entropy-regularized natural policy gradient with linear function approximation.
\newblock \emph{SIAM Journal on Optimization}, 34\penalty0 (3):\penalty0 2729--2755, 2024.

\bibitem[Cen et~al.(2022)Cen, Cheng, Chen, Wei, and Chi]{cen2022fast}
Shicong Cen, Chen Cheng, Yuxin Chen, Yuting Wei, and Yuejie Chi.
\newblock Fast global convergence of natural policy gradient methods with entropy regularization.
\newblock \emph{Operations Research}, 70\penalty0 (4):\penalty0 2563--2578, 2022.

\bibitem[Chen \& Xu(2020)Chen and Xu]{chen2020deep}
Lin Chen and Sheng Xu.
\newblock Deep neural tangent kernel and laplace kernel have the same rkhs.
\newblock \emph{arXiv preprint arXiv:2009.10683}, 2020.

\bibitem[Cho \& Saul(2009)Cho and Saul]{cho2009kernel}
Youngmin Cho and Lawrence Saul.
\newblock Kernel methods for deep learning.
\newblock \emph{Advances in neural information processing systems}, 22, 2009.

\bibitem[Ding et~al.(2024)Ding, Hu, Jiang, Li, Wang, and Yao]{ding2024random}
Liang Ding, Tianyang Hu, Jiahang Jiang, Donghao Li, Wenjia Wang, and Yuan Yao.
\newblock Random smoothing regularization in kernel gradient descent learning.
\newblock \emph{Journal of Machine Learning Research}, 25\penalty0 (284):\penalty0 1--88, 2024.

\bibitem[Duan et~al.(2024)Duan, Wang, and Wainwright]{duan2024optimal}
Yaqi Duan, Mengdi Wang, and Martin~J Wainwright.
\newblock Optimal policy evaluation using kernel-based temporal difference methods.
\newblock \emph{The Annals of Statistics}, 52\penalty0 (5):\penalty0 1927--1952, 2024.

\bibitem[Farahmand et~al.(2016)Farahmand, Ghavamzadeh, Szepesv{\'a}ri, and Mannor]{farahmand2016regularized}
Amir-massoud Farahmand, Mohammad Ghavamzadeh, Csaba Szepesv{\'a}ri, and Shie Mannor.
\newblock Regularized policy iteration with nonparametric function spaces.
\newblock \emph{Journal of Machine Learning Research}, 17\penalty0 (139):\penalty0 1--66, 2016.

\bibitem[Feng et~al.(2020)Feng, Ren, Tang, and Liu]{feng2020accountable}
Yihao Feng, Tongzheng Ren, Ziyang Tang, and Qiang Liu.
\newblock Accountable off-policy evaluation with kernel bellman statistics.
\newblock In \emph{International Conference on Machine Learning}, pp.\  3102--3111. PMLR, 2020.

\bibitem[Geer(2000)]{geer2000empirical}
Sara~A Geer.
\newblock \emph{Empirical Processes in M-estimation}, volume~6.
\newblock Cambridge university press, 2000.

\bibitem[Grunewalder et~al.(2012)Grunewalder, Lever, Baldassarre, Pontil, and Gretton]{grunewalder2012modelling}
Steffen Grunewalder, Guy Lever, Luca Baldassarre, Massi Pontil, and Arthur Gretton.
\newblock Modelling transition dynamics in mdps with rkhs embeddings.
\newblock \emph{arXiv preprint arXiv:1206.4655}, 2012.

\bibitem[Hamm \& Steinwart(2021)Hamm and Steinwart]{hamm2021adaptive}
Thomas Hamm and Ingo Steinwart.
\newblock Adaptive learning rates for support vector machines working on data with low intrinsic dimension.
\newblock \emph{The Annals of Statistics}, 49\penalty0 (6):\penalty0 3153--3180, 2021.

\bibitem[Hofmann et~al.(2008)Hofmann, Sch{\"o}lkopf, and Smola]{hofmann2008kernel}
Thomas Hofmann, Bernhard Sch{\"o}lkopf, and Alexander~J Smola.
\newblock Kernel methods in machine learning.
\newblock 2008.

\bibitem[Hu et~al.(2021)Hu, Wang, Lin, and Cheng]{hu2021regularization}
Tianyang Hu, Wenjia Wang, Cong Lin, and Guang Cheng.
\newblock Regularization matters: A nonparametric perspective on overparametrized neural network.
\newblock In \emph{International Conference on Artificial Intelligence and Statistics}, pp.\  829--837. PMLR, 2021.

\bibitem[Jacot et~al.(2018)Jacot, Gabriel, and Hongler]{jacot2018neural}
Arthur Jacot, Franck Gabriel, and Cl{\'e}ment Hongler.
\newblock Neural tangent kernel: Convergence and generalization in neural networks.
\newblock \emph{Advances in neural information processing systems}, 31, 2018.

\bibitem[Kakade \& Langford(2002)Kakade and Langford]{kakade2002approximately}
Sham Kakade and John Langford.
\newblock Approximately optimal approximate reinforcement learning.
\newblock In \emph{Proceedings of the nineteenth international conference on machine learning}, pp.\  267--274, 2002.

\bibitem[Kakade(2001)]{kakade2001natural}
Sham~M Kakade.
\newblock A natural policy gradient.
\newblock \emph{Advances in neural information processing systems}, 14, 2001.

\bibitem[Konda \& Tsitsiklis(1999)Konda and Tsitsiklis]{konda1999actor}
Vijay Konda and John Tsitsiklis.
\newblock Actor-critic algorithms.
\newblock \emph{Advances in neural information processing systems}, 12, 1999.

\bibitem[Koppel et~al.(2020)Koppel, Warnell, Stump, Stone, and Ribeiro]{koppel2020policy}
Alec Koppel, Garrett Warnell, Ethan Stump, Peter Stone, and Alejandro Ribeiro.
\newblock Policy evaluation in continuous mdps with efficient kernelized gradient temporal difference.
\newblock \emph{IEEE Transactions on Automatic Control}, 66\penalty0 (4):\penalty0 1856--1863, 2020.

\bibitem[Krieg \& Sonnleitner(2024)Krieg and Sonnleitner]{krieg2024random}
David Krieg and Mathias Sonnleitner.
\newblock Random points are optimal for the approximation of sobolev functions.
\newblock \emph{IMA Journal of Numerical Analysis}, 44\penalty0 (3):\penalty0 1346--1371, 2024.

\bibitem[K{\"u}hn(2011)]{kuhn2011covering}
Thomas K{\"u}hn.
\newblock Covering numbers of gaussian reproducing kernel hilbert spaces.
\newblock \emph{Journal of Complexity}, 27\penalty0 (5):\penalty0 489--499, 2011.

\bibitem[Lakshminarayanan \& Szepesvari(2018)Lakshminarayanan and Szepesvari]{lakshminarayanan2018linear}
Chandrashekar Lakshminarayanan and Csaba Szepesvari.
\newblock Linear stochastic approximation: How far does constant step-size and iterate averaging go?
\newblock In \emph{International conference on artificial intelligence and statistics}, pp.\  1347--1355. PMLR, 2018.

\bibitem[Lin \& Zhou(2018)Lin and Zhou]{lin2018distributed}
Shao-Bo Lin and Ding-Xuan Zhou.
\newblock Distributed kernel-based gradient descent algorithms.
\newblock \emph{Constructive Approximation}, 47\penalty0 (2):\penalty0 249--276, 2018.

\bibitem[Liu et~al.(2019)Liu, Cai, Yang, and Wang]{liu2019neural}
Boyi Liu, Qi~Cai, Zhuoran Yang, and Zhaoran Wang.
\newblock Neural trust region/proximal policy optimization attains globally optimal policy.
\newblock \emph{Advances in neural information processing systems}, 32, 2019.

\bibitem[Liu et~al.(2024)Liu, Li, and Wei]{liu2024elementary}
Jiacai Liu, Wenye Li, and Ke~Wei.
\newblock Elementary analysis of policy gradient methods.
\newblock \emph{arXiv preprint arXiv:2404.03372}, 2024.

\bibitem[Maei et~al.(2009)Maei, Szepesvari, Bhatnagar, Precup, Silver, and Sutton]{maei2009convergent}
Hamid Maei, Csaba Szepesvari, Shalabh Bhatnagar, Doina Precup, David Silver, and Richard~S Sutton.
\newblock Convergent temporal-difference learning with arbitrary smooth function approximation.
\newblock \emph{Advances in neural information processing systems}, 22, 2009.

\bibitem[Mei et~al.(2020)Mei, Xiao, Szepesvari, and Schuurmans]{mei2020global}
Jincheng Mei, Chenjun Xiao, Csaba Szepesvari, and Dale Schuurmans.
\newblock On the global convergence rates of softmax policy gradient methods.
\newblock In \emph{International conference on machine learning}, pp.\  6820--6829. PMLR, 2020.

\bibitem[Munos \& Szepesv{\'a}ri(2008)Munos and Szepesv{\'a}ri]{munos2008finite}
R{\'e}mi Munos and Csaba Szepesv{\'a}ri.
\newblock Finite-time bounds for fitted value iteration.
\newblock \emph{Journal of Machine Learning Research}, 9\penalty0 (5), 2008.

\bibitem[Nemirovsky \& Yudin(1985)Nemirovsky and Yudin]{Nemirovsky1985}
A.~S. Nemirovsky and D.~B. Yudin.
\newblock Problem complexity and method efficiency in optimization.
\newblock \emph{SIAM Review}, 27\penalty0 (2):\penalty0 264--265, 1985.
\newblock \doi{10.1137/1027074}.

\bibitem[Nesterov(2013)]{nesterov2013introductory}
Yurii Nesterov.
\newblock \emph{Introductory lectures on convex optimization: A basic course}, volume~87.
\newblock Springer Science \& Business Media, 2013.

\bibitem[Neu et~al.(2017)Neu, Jonsson, and G{\'o}mez]{neu2017unified}
Gergely Neu, Anders Jonsson, and Vicen{\c{c}} G{\'o}mez.
\newblock A unified view of entropy-regularized markov decision processes.
\newblock \emph{arXiv preprint arXiv:1705.07798}, 2017.

\bibitem[Neuberger(2009)]{neuberger2009sobolev}
John Neuberger.
\newblock \emph{Sobolev gradients and differential equations}.
\newblock Springer Science \& Business Media, 2009.

\bibitem[Ormoneit \& Sen(2002)Ormoneit and Sen]{ormoneit2002kernel}
Dirk Ormoneit and {\'S}aunak Sen.
\newblock Kernel-based reinforcement learning.
\newblock \emph{Machine learning}, 49\penalty0 (2):\penalty0 161--178, 2002.

\bibitem[Raskutti et~al.(2014)Raskutti, Wainwright, and Yu]{raskutti2014early}
Garvesh Raskutti, Martin~J Wainwright, and Bin Yu.
\newblock Early stopping and non-parametric regression: an optimal data-dependent stopping rule.
\newblock \emph{The Journal of Machine Learning Research}, 15\penalty0 (1):\penalty0 335--366, 2014.

\bibitem[Rieger(2008)]{rieger2008sampling}
Christian Rieger.
\newblock \emph{Sampling inequalities and applications}.
\newblock Georg-August-Universitaet Goettingen (Germany), 2008.

\bibitem[Schulman et~al.(2015)Schulman, Levine, Abbeel, Jordan, and Moritz]{schulman2015trust}
John Schulman, Sergey Levine, Pieter Abbeel, Michael Jordan, and Philipp Moritz.
\newblock Trust region policy optimization.
\newblock In \emph{International conference on machine learning}, pp.\  1889--1897. PMLR, 2015.

\bibitem[Schulman et~al.(2017)Schulman, Wolski, Dhariwal, Radford, and Klimov]{schulman2017proximal}
John Schulman, Filip Wolski, Prafulla Dhariwal, Alec Radford, and Oleg Klimov.
\newblock Proximal policy optimization algorithms.
\newblock \emph{arXiv preprint arXiv:1707.06347}, 2017.

\bibitem[Shani et~al.(2020)Shani, Efroni, and Mannor]{shani2020adaptive}
Lior Shani, Yonathan Efroni, and Shie Mannor.
\newblock Adaptive trust region policy optimization: Global convergence and faster rates for regularized mdps.
\newblock In \emph{Proceedings of the AAAI conference on artificial intelligence}, volume~34, pp.\  5668--5675, 2020.

\bibitem[Srikant \& Ying(2019)Srikant and Ying]{srikant2019finite}
Rayadurgam Srikant and Lei Ying.
\newblock Finite-time error bounds for linear stochastic approximation andtd learning.
\newblock In \emph{Conference on learning theory}, pp.\  2803--2830. PMLR, 2019.

\bibitem[Sutton(1988)]{sutton1988learning}
Richard~S Sutton.
\newblock Learning to predict by the methods of temporal differences.
\newblock \emph{Machine learning}, 3\penalty0 (1):\penalty0 9--44, 1988.

\bibitem[Sutton et~al.(1998)Sutton, Barto, et~al.]{sutton1998reinforcement}
Richard~S Sutton, Andrew~G Barto, et~al.
\newblock \emph{Reinforcement learning: An introduction}, volume~1.
\newblock MIT press Cambridge, 1998.

\bibitem[Sutton et~al.(1999)Sutton, McAllester, Singh, and Mansour]{sutton1999policy}
Richard~S Sutton, David McAllester, Satinder Singh, and Yishay Mansour.
\newblock Policy gradient methods for reinforcement learning with function approximation.
\newblock \emph{Advances in neural information processing systems}, 12, 1999.

\bibitem[Tuo \& Jeff~Wu(2016)Tuo and Jeff~Wu]{tuo2016theoretical}
Rui Tuo and CF~Jeff~Wu.
\newblock A theoretical framework for calibration in computer models: Parametrization, estimation and convergence properties.
\newblock \emph{SIAM/ASA Journal on Uncertainty Quantification}, 4\penalty0 (1):\penalty0 767--795, 2016.

\bibitem[Tuo et~al.(2020)Tuo, Wang, and Jeff~Wu]{tuo2020improved}
Rui Tuo, Yan Wang, and CF~Jeff~Wu.
\newblock On the improved rates of convergence for mat{\'e}rn-type kernel ridge regression with application to calibration of computer models.
\newblock \emph{SIAM/ASA Journal on Uncertainty Quantification}, 8\penalty0 (4):\penalty0 1522--1547, 2020.

\bibitem[Utreras(1988)]{utreras1988convergence}
Florencio~I Utreras.
\newblock Convergence rates for multivariate smoothing spline functions.
\newblock \emph{Journal of approximation theory}, 52\penalty0 (1):\penalty0 1--27, 1988.

\bibitem[Wang et~al.(2019)Wang, Cai, Yang, and Wang]{wang2019neural}
Lingxiao Wang, Qi~Cai, Zhuoran Yang, and Zhaoran Wang.
\newblock Neural policy gradient methods: Global optimality and rates of convergence.
\newblock \emph{arXiv preprint arXiv:1909.01150}, 2019.

\bibitem[Wendland(2004)]{wendland2004scattered}
Holger Wendland.
\newblock \emph{Scattered data approximation}, volume~17.
\newblock Cambridge university press, 2004.

\bibitem[Zhang et~al.(2020)Zhang, Koppel, Zhu, and Basar]{zhang2020global}
Kaiqing Zhang, Alec Koppel, Hao Zhu, and Tamer Basar.
\newblock Global convergence of policy gradient methods to (almost) locally optimal policies.
\newblock \emph{SIAM Journal on Control and Optimization}, 58\penalty0 (6):\penalty0 3586--3612, 2020.

\bibitem[Zhou(2008)]{zhou2008derivative}
Ding-Xuan Zhou.
\newblock Derivative reproducing properties for kernel methods in learning theory.
\newblock \emph{Journal of computational and Applied Mathematics}, 220\penalty0 (1-2):\penalty0 456--463, 2008.

\end{thebibliography}
\bibliographystyle{iclr2026_conference}

\appendix
\section{Appendix}
\section{Proof of Lemma \ref{lem:Q_pi_RKHS}}
\begin{proof}
    Form the definition of $Q^{\pi}$ and triangle inequality, we have
    \begin{equation}
        \label{eq:Q_pi_RKHS_1}
        \|Q^\pi\|_\CalH\leq \sum_{t=0}^\infty\gamma^t\|\E[r(s_t,a_t)|s_0,a_0]\|_{\CalH}.
    \end{equation}
    So we only need to prove $\|\E[r(s_t,a_t)|s_0=\cdot,a_0=\cdot]\|_{\CalH}\leq C_1R$ for some $C_1<\infty$ for all $t$. For the case $t=0$, we use Assumption \ref{assump:r_pi_P_RKHS} to have 
    \[|\E[r(s_0,a_0)|s_0=\cdot,a_0=\cdot]\|_{\CalH}=\|r\|_{\CalH}\leq R.\]
    For general $t$, we rewrite $\E[r(s_t,a_t)|s_0=\cdot,a_0=\cdot]$ in an integral form
    \begin{align}
        &\E[r(s_t,a_t)|s_0,a_0]\nonumber\\
        =&\int r(s_t,a_t)\left(\prod_{\tau=1}^t\pi(a_\tau|s_\tau)P(s_\tau|s_{\tau-1},a_\tau-1)\right)P(s_1|s_0,a_0)\pi(a_0|s_0) ds_1da_1\cdots ds_tda_t\nonumber\\
        =&\int r(s_t,a_t) P\left((s_1)\to(s_t,a_t)\right)P(s_1|s_0,a_0)\pi(a_0|s_0)\label{eq:Q_pi_RKHS_2}
    \end{align}
    where $P\left((s_1)\to(s_t,a_t)\right)$ is the distribution of $(s_t,a_t)$ conditioned on $s_1$. Let $\{\phi_i\}$ be the eigenfunction of $K$ then the RKHS norms of $P(s|\cdot,\cdot)$ and $\pi$ are then
    \begin{equation}
        \label{eq:Q_pi_RKHS_3}
        \|P(s|\cdot,\cdot)\|_\CalH^2=\sum_ip_i(s)^2\leq R,\quad \|\pi\|_{\CalH}^2=\sum\pi_i^2\leq R
    \end{equation}
    where $p_i(s)=\langle P(s|\cdot,\cdot),\phi_i\rangle_\CalH$ and $\pi_i=\langle\pi,\phi_i\rangle_\CalH$.

    The coefficient regarding $\E[r(s_t,a_t)|s_0,a_0]$ projected onto $\phi_i$ is then
    \begin{equation}
        \label{eq:Q_pi_RKHS_4}
        \begin{aligned}
        E_i=&\langle \E[r(s_t,a_t)|s_0,a_0],\phi_i\rangle_\CalH=\int r(s_t,a_t) P\left((s_1)\to(s_t,a_t)\right)\langle P(s_1|s_0,a_0)\pi(a_0|s_0),\phi_i\rangle_\CalH\\
        \leq & \max_{(s,a)\in\CalS\times\CalA}r(s,a) \max_{s\in\CalS}\langle P(s|s_0,a_0)\pi(a_0|s_0),\phi_i\rangle_\CalH\\
        \leq & \max_{\omega\in\CalS\times\CalA}K(\omega,\omega) R \max_{s\in\CalS}\langle P(s|s_0,a_0)\pi(a_0|s_0),\phi_i\rangle_\CalH
    \end{aligned}
    \end{equation}
    where the last line is because $\max_\omega r(\omega)=\langle r,K(\cdot,\omega)\rangle_\CalH\leq \|K(\cdot,\omega)\|_\CalH\|r\|_{\CalH}\leq \max_\omega K(\omega,\omega) R$.

    From \eqref{eq:Q_pi_RKHS_4}, we only need to calculate the RKHS norm of function $P(s|s_0,a_0)\pi(a_0|s_0)$. From \eqref{eq:Q_pi_RKHS_3}, we can derive that
    \begin{equation}
        \label{eq:Q_pi_RKHS_5}
        \begin{aligned}
            \|P(s|s_0,a_0)\pi(a_0|s_0)\|^2_\CalH=&\|\left(\sum_ip_i(s)\phi_i\right)\left(\sum_i\pi_i\phi_i\right)\|_\CalH^2\\
            =& \sum_i p_i(s)^2\pi_i^2\leq \left(\sup_s\sup_i p_i(s)^2\right)\|\pi\|_{\CalH}^2\leq R^2\|\pi\|_{\CalH}^2.
        \end{aligned}
    \end{equation}
  By combining \eqref{eq:Q_pi_RKHS_3}~\eqref{eq:Q_pi_RKHS_5}, we can conclude that
  \begin{equation}
    \label{eq:Q_pi_RKHS_6}
      \|\E[r(s_t,a_t)|s_0=\cdot,a_0=\cdot]\|_{\CalH}\leq C \|\pi\|_{\CalH}^2.
  \end{equation}
  Substitute \eqref{eq:Q_pi_RKHS_6} into \eqref{eq:Q_pi_RKHS_1}, we can have the final result.
\end{proof}

\section{Proof of Proposition \ref{prop:KRR_TD_representer_thm}}
\begin{proof}[Proof of \eqref{eq:KRR_TD_closed_form}]

We first briefly introduce the concept of covariance operators, which is crucial in the proof. We adopt the shorthand notation $\Phi(\omega)=K(\omega,\cdot)$ and $\mu^{\pi}(\omega,\omega')=\mu^{\pi}((s,a)),(s',a')=\mu(s)\pi(s|a)P(s'|s,a)\pi(a'|s')$. Then the  covariance operator  $C_{\omega_0,\omega_0}$  and cross-covariance operator $C_{\omega_0,\omega_1}$ on the RKHS $\CalH$ are defined as:
\begin{align*}
    &C_{\omega_0,\omega_0}=\E_{(S,A) \sim \mu(s)\pi(a|s)}[\Phi(S,A)\otimes \Phi(S,A)],\\
    &C_{\omega_0,\omega_1}=\E_{((S,A),(S',A'))\sim\mu^{\pi}}[\Phi(S,A)\otimes \Phi({S',A'})]
\end{align*}
where $f\otimes g$ denote the tensor product between functions $f$ and $g$. Accordingly, for any $f_0,f_1\in\CalH$, $C_{\omega_0,\omega_0}$ and $C_{\omega_0,\omega_1}$ satisfy 
\begin{align*}
    &\langle f_0, C_{\omega_0,\omega_0} f_1\rangle_\CalH=\E_{(S,A) \sim \mu(s)\pi(a|s)}[f_0(S,A)f_1(S,A)],\\
    &\langle f_0, C_{\omega_0,\omega_1} f_1\rangle_\CalH=\E_{((S,A),(S',A'))\sim\mu^{\pi}}[f_0(S,A)f_1(S',A')]].
\end{align*}
We can also define the empirical version of operators $C_{\omega_0,\omega_0}$ and $C_{\omega_0,\omega_1}$ as follows:
\[\hat{C}_{\omega_0,\omega_0}=\frac{1}{n}\sum_{i=1}^{n}\Phi({\omega^{(i)}_0})\otimes \Phi({\omega^{(i)}_0}),\quad \hat{C}^{\pi}_{\omega_0,\omega_1}=\frac{1}{n}\sum_{i=1}^{n}\Phi({\omega^{(i)}_0})\otimes \Phi({\omega^{(i)}_1}).\]

With the concept of covariance operator, we can have the following lemma which can yield a specific formulation for the solution of \eqref{eq:KRR_TD}:
\begin{lemma}
\label{lem:KRR_TD_solution}
    Solution to \eqref{eq:KRR_TD} is equivalent to the following equation:
    \begin{equation}
        \label{eq:KRR_TD_solution}
        \hat{C}_{\omega_0,\omega_0}\hat{Q}^{\pi}-\left[ \hat{C}_{\omega_0,\omega_0}r+\gamma\hat{C}_{\omega_0,\omega_1} \hat{Q}^{\pi}\right]+\lambda\hat{Q}^{\pi}=0.
    \end{equation}
\end{lemma}

From Lemma \ref{lem:KRR_TD_solution}, we can prove that the solution $\hat{Q}^{\pi}$ resides in a finite-dimensional space:
\begin{lemma}
    \label{lem:KRR_TD_finite_dim}
    Let $\hat{\CalH}=\{K(\omega_i,\cdot),i=1,\cdots,n\}$ be the finite dimensional space spanned by kernel functions. We have $\hat{Q}^{\pi}\in\hat{\CalH}$.
\end{lemma}
 From Lemma \ref{lem:KRR_TD_finite_dim}, we know that $\hat{Q}^{\pi}$ must be of the form:
\begin{equation}
\label{eq:KRR_TD_error_decomposition_1}
    \hat{Q}^{\pi}=\sum_{i=1}^n K(\cdot,\omega_i)b_i=K(\cdot,\Bomega_0)\bold{b}.
\end{equation}
Substitute \eqref{eq:KRR_TD_error_decomposition_1} into \eqref{eq:KRR_TD_solution}, togethe with the definition of empirical operators $\hat{C}_{\omega_0,\omega_0}$ and $\hat{C}_{\omega_0,\omega_1}$, we then have
\begin{align}
\label{eq:KRR_TD_error_decomposition_2}
    &\frac{1}{n}\sum_{i=1}^n\Phi(\omega_0^{(i)})\sum_{j=1}^nb_j K(\omega_0^{(i)},\omega_0^{(j)})-\frac{1}{n}\sum_{i=1}^n\Phi(\omega_0^{(i)})\left[r(\omega_0^{(i)})+\gamma \sum_{j=1}^n b_jK(\omega_1^{(i)},\omega_0^{(j)})\right]\\
    &+\lambda\sum_{i=1}^nb_i\Phi(\omega_0^{(i)})=0.\nonumber
\end{align}
 Rearrange \eqref{eq:KRR_TD_error_decomposition_2} and write it in vector form, we have
\begin{equation}
    \Phi(\Bomega_0)\left[\BK\bold{b}+\lambda n\bold{b}-\gamma \bold{C}\bold{b}\right]=\Phi(\Bomega_0)\Br.
\end{equation}
So, we have
\[\bold{b}=\left[\BK+\lambda n {\rm I}-\gamma \bold{C}\right]^{-1}\Br.\]
\end{proof}

\subsection{Proof of Lemma \ref{lem:KRR_TD_solution}}
\begin{proof}
    Define the functional 
    \begin{equation}
    \label{eq:KRR_TD_solution_0}
        \CalJ[f]=\|\hat{C}_{\omega_0,\omega_0} f-\left(\hat{C}_{\omega_0,\omega_0}r+\gamma\hat{C}_{\omega_0,\omega_1}Q^{\pi}\right)\|_{n}^2+\lambda \|f\|_{\CalH}^2.
    \end{equation}
    The  functional $\CalJ$ is Fr\'echet differentiable obviously with RKHS derivative $ \nabla_f\CalJ[f]$ as follows:
    \begin{align}
        \label{KRR_TD_solution_1}
        &\frac{1}{2}\partial_\delta\CalJ[f+\delta u]\big|_{\delta=0}\\
        =&\langle \frac{1}{n}\sum_{i=1}^{n}\left(f(\omega^{(i)}_0)-r(\omega^{(i)}_0)-\gamma\hat{Q}^{\pi}(\omega^{(i)}_1)\right)K(\omega^{(i)}_0,\cdot)+\lambda f,u\rangle_\CalH,\nonumber\\
        =&\langle \nabla_f\CalJ[f],u\rangle_\CalH,\quad \forall u\in\CalH.
    \end{align}
    From the definition of empirical covariance operators, we can notice that
    \begin{align}
        &\frac{1}{n}\sum_{i=1}^{n}\left(f(\omega^{(i)}_0)-r(\omega^{(i)}_0)\right)K(\omega^{(i)}_0,\cdot)=\hat{C}_{s_0,s_0}[f-r], \label{KRR_TD_solution_2}\\
        &\frac{1}{n}\sum_{i=1}^{n}\hat{Q}^{\pi}(\omega^{(i)}_1)K(\omega^{(i)}_0,\cdot)=\hat{C}_{\omega_0,\omega_1}\hat{Q}^{\pi}  \label{KRR_TD_solution_3}.
    \end{align}
    Substitute \eqref{KRR_TD_solution_2} and \eqref{KRR_TD_solution_3} into \eqref{KRR_TD_solution_1}, we then have:
    \begin{equation}
        \label{KRR_TD_solution_4}
       \nabla_f\CalJ[f]= \hat{C}_{\omega_0,\omega_0}f- \left[\hat{C}_{\omega_0,\omega_0}r+\gamma \hat{C}_{\omega_0,\omega_1}\hat{Q}^{\pi}\right]+\lambda f.
    \end{equation}
    According to definition, $\hat{Q}^{\pi}$ is also the minimizer of $\CalJ$. Substitute $f=\hat{Q}^{\pi}$ into \eqref{KRR_TD_solution_4}, we can conclude that the minimization of \eqref{eq:KRR_TD_solution_0} is equivalent to \eqref{eq:KRR_TD}.
\end{proof}

\subsection{Proof of Lemma \ref{lem:KRR_TD_finite_dim}}
\begin{proof}
    Rearrange \eqref{eq:KRR_TD_solution} as follows
    \begin{equation}
        \label{eq:KRR_TD_finite_dim_1}
       \lambda \hat{Q}^{\pi}=\hat{C}_{\omega_0,\omega_0}r+\gamma\hat{C}_{\omega_0,\omega_1} \hat{Q}^{\pi}- \hat{C}_{\omega_0,\omega_0}\hat{Q}^{\pi}.
    \end{equation}
    According to \eqref{KRR_TD_solution_2} and \eqref{KRR_TD_solution_3}, the right hand side of \eqref{eq:KRR_TD_finite_dim_1} can be written as
    \begin{align*}
        &\hat{C}_{\omega_0,\omega_0}\left[r-\hat{Q}^{\pi}\right]+\gamma\hat{C}_{\omega_0,\omega_1} \hat{Q}^{\pi}\\
        =&\frac{1}{n}K(\cdot,\omega^{(i)}_0)\left[r(\omega^{(i)}_0)- \hat{Q}^{\pi}(\omega^{(i)}_0)+\hat{Q}^{\pi}(\omega^{(i)}_1)\right]\in\hat{\CalH}.
    \end{align*}
    The right hand side of  \eqref{eq:KRR_TD_finite_dim_1} resides in $\hat{\CalH}$. So, the left hand side $\hat{Q}^{\pi}\in\hat{\CalH}$.
\end{proof}

\section{Proof of Proposition \ref{prop:KRR_TD_error_decomposition}}
\begin{proof}[Proof.]
    According to the definition of covariance operator, the Bellman equation \eqref{eq:bellman} is equivalent to to follow form:
    \begin{equation}
        \label{eq:error_decomposition_1}
        C_{\omega_0,\omega_0}Q^{\pi}=C_{\omega_0,\omega_0}r+\gamma C_{\omega_0,\omega_1}Q^{\pi}.
    \end{equation}
    Substitute \eqref{eq:error_decomposition_1}  into Lemma \ref{lem:KRR_TD_solution}, we can have 
        \begin{equation}
        \label{eq:error_decomposition_2}
        \left[\hat{C}_{\omega_0,\omega_0}+\lambda{\rm I}\right]\hat{Q}^{\pi}-C_{\omega_0,\omega_0}Q^{\pi}=\left[\hat{C}_{\omega_0,\omega_0}-C_{\omega_0,\omega_0}\right]r+\gamma\left[\hat{C}_{\omega_0,\omega_1}\hat{Q}^{\pi}-C_{\omega_0,\omega_1}Q^{\pi}\right].
    \end{equation}
    Recall that $\CalD^{\pi}=\hat{Q}^{\pi}-Q^{\pi}$,  we can rewrite \eqref{eq:error_decomposition_2} as
    \begin{align}
        \label{eq:error_decomposition_3}
        &\left[\hat{C}_{\omega_0,\omega_0}+\lambda{\rm I}-\gamma\hat{C}_{\omega_0,\omega_1}\right]\CalD^{\pi}\\
        =&\left[\hat{C}_{\omega_0,\omega_0}+\lambda{\rm I}-C_{\omega_0,\omega_0}\right]\left(r-Q^{\pi}\right)-\lambda r+\gamma\left[\hat{C}_{\omega_0,\omega_1}-C_{\omega_0,\omega_1}\right]Q^{\pi}\nonumber\\
        =&\left[\hat{C}_{\omega_0,\omega_0}+\lambda{\rm I}\right]\left(r-Q^{\pi}\right)+\gamma \hat{C}_{s_0,s_1}Q^{\pi}-\lambda r\nonumber\\
        &+\left(C_{\omega_0,\omega_0}\left(Q^{\pi}-r\right)-\gamma C_{\omega_0,\omega_1}Q^{\pi}\right)\nonumber\\
        =&\left[\hat{C}_{\omega_0,\omega_0}+\lambda{\rm I}\right]\left(r-Q^{\pi}\right)+\gamma \hat{C}_{\omega_0,\omega_1}Q^{\pi}-\lambda r\label{eq:error_decomposition_4}
    \end{align}
    where the last line is from \eqref{eq:error_decomposition_1}. Taking the Hilbert space inner product with $\CalD^{\pi}$ in \eqref{eq:error_decomposition_3} and \eqref{eq:error_decomposition_4}:
    \begin{align}
        &\langle\CalD^{\pi},\left[\hat{C}_{\omega_0,\omega_0}+\lambda {\rm I}-\gamma\hat{C}_{\omega_0,\omega_1}\right]\CalD^{\pi}\rangle_\CalH\label{eq:error_decomposition_5}\\
        =& \langle\CalD^{\pi},\left[\hat{C}_{\omega_0,\omega_0}+\lambda{\rm I}\right]\left(r-Q^{\pi}\right)+\gamma \hat{C}_{\omega_0,\omega_1}Q^{\pi}-\lambda r \rangle_\CalH\label{eq:error_decomposition_6}.
    \end{align}
According to the definition of empirical covariance operators, \eqref{eq:error_decomposition_5} can be further rewritten as
\begin{align}
\label{eq:error_decomposition_7}
    &\langle\CalD^{\pi},\left[\hat{C}_{\omega_0,\omega_0}+\lambda {\rm I}-\gamma\hat{C}_{\omega_0,\omega_1}\right]\CalD^{\pi}\rangle_\CalH\\
    =&\frac{1}{n}\sum_{i=1}^n\CalD^{\pi}(\omega_0^{(i)})^2-\frac{\gamma}{n}\sum_{i=1}^n\CalD^{\pi}(\omega_0^{(i)})\CalD^{\pi}(\omega_1^{(i)})+\lambda\|\CalD^{\pi}\|^2_\CalH\nonumber
\end{align}
and \eqref{eq:error_decomposition_6} can be further rewritten as
\begin{align}
    \label{eq:error_decomposition_8}
    &\langle\CalD^{\pi},\left[\hat{C}_{\omega_0,\omega_0}+\lambda_K{\rm I}\right]\left(r-Q^{\pi}\right)+\gamma \hat{C}_{\omega_0,\omega_1}Q^{\pi}-\lambda r \rangle_\CalH\\
    =&\langle\CalD^{\pi},\hat{C}_{\omega_0,\omega_0}\left(r-Q^{\pi}\right)+\gamma \hat{C}_{\omega_0,\omega_1}Q^{\pi}\rangle_\CalH-\lambda \langle\CalD^{\pi},Q^{\pi}\rangle_\CalH\nonumber.
\end{align}
Because \eqref{eq:error_decomposition_7} and \eqref{eq:error_decomposition_8} are equals, by rearranging their terms, we have
\begin{align}
    \label{eq:error_decomposition_9}
    & \frac{1}{n}\sum_{i=1}^n\CalD^{\pi}(\omega_0^{(i)})^2-\frac{\gamma}{n}\sum_{i=1}^n\CalD^{\pi}(\omega_0^{(i)})\CalD^{\pi}(\omega_1^{(i)})\\
    =&\langle\CalD^{\pi},\hat{C}_{\omega_0,\omega_0}\left(r-Q^{\pi}\right)+\gamma \hat{C}_{\omega_0,\omega_1}Q^{\pi}\rangle_\CalH-\lambda \langle\CalD^{\pi},\hat{Q}^{\pi}\rangle_\CalH\label{eq:error_decomposition_10}.
\end{align}
We use the definitions of empirical and population covariance operators again, the first term of \eqref{eq:error_decomposition_10} is
\begin{align}
\label{eq:error_decomposition_11}
    &\langle\CalD^{\pi},\hat{C}_{\omega_0,\omega_0}\left(r-Q^{\pi}\right)+\gamma \hat{C}_{\omega_0,\omega_1}Q^{\pi}\rangle_\CalH\\
    =&\frac{1}{n}\sum_{i=1}^n\CalD^{\pi}(\omega_0^{(i)})\left(r(\omega_0^{(i)})-Q^{\pi}(\omega_0^{(i)})+\gamma Q^{\pi}(\omega_1^{(i)})\right)
    =\frac{1}{n}\sum_{i=1}^n\varepsilon_i\CalD^{\pi}(\omega_0^{(i)})\nonumber.
\end{align}
Substitute \eqref{eq:error_decomposition_11} into \eqref{eq:error_decomposition_10}, we can have the final result.
\end{proof}

\section{Proof of Theorem \ref{thm:convergence_Kernel_TD}}

In order to proof  Theorem \ref{thm:convergence_Kernel_TD}, we first need the following theorem, which we will provide its proof in Appendix \ref{Sec:pf_convergence_KRR}.
\begin{theorem}
    \label{thm:convergence_KRR}
         Let 
         $\lambda =\CalO((1-c\gamma)^{\frac{\beta}{2+2\beta}}n^{-\frac{1}{2+2\beta}}|\log n|^{\frac{\kappa}{1+\beta}})$.
         Under Assumptions~\ref{assump:distribution_bounded}, \ref{assump:r_pi_P_RKHS}, \ref{assump:entropy}, and \ref{assump:dist_mu_gamma} we can have
        \begin{equation}
        \label{eq:convergence_KRR}
        \begin{aligned}
             &\sqrt{\frac{1}{n}\sum_{i=1}^n\left|\CalD^{\pi}(\omega_0^{(i)})\right|^2}\leq \CalO_p\left((1-c\gamma)^{-\frac{2+\beta}{2+2\beta}}n^{-\frac{1}{2+2\beta}}|\log n|^{\frac{\kappa}{1+\beta}}\right)\|Q^\pi\|_{\CalH},\\
    &\|\hat{Q}^\pi\|_{\CalH}\leq \CalO_p(1)\|Q^\pi\|_{\CalH}.
        \end{aligned}
        \end{equation}
    \end{theorem}

\begin{proof}[Proof of Theorem \ref{thm:convergence_Kernel_TD}:]
    By triangle inequality and Theorem \ref{thm:convergence_KRR}, we have
    \begin{equation}
        \label{eq:convergence_Kernel_TD_0}
        \begin{aligned}
          \|f_t-Q^\pi\|_n\leq& \|f_t-\hat{Q}^\pi\|_n+ \|\hat{Q}^\pi-Q^\pi\|_n\\
          \leq &\|f_t-\hat{Q}^\pi\|_n+ \CalO_p\left((1-c\gamma)^{-\frac{2+\beta}{2+2\beta}}n^{-\frac{1}{2+2\beta}}|\log n|^{\frac{\kappa}{1+\beta}}\right),
    \end{aligned}
    \end{equation}
    and
    \begin{align}
    \label{eq:convergence_Kernel_TD_01}
        &\|f_t\|_\CalH\leq \|\hat{Q}^\pi\|_\CalH+\|f_t-\hat{Q}^\pi\|_\CalH+ \|\hat{Q}^\pi-Q^\pi\|_\CalH\leq \CalO_p(1)\|Q^\pi\|_\CalH+\|f_t-\hat{Q}^\pi\|_\CalH.
    \end{align}
    Therefore, we only need to prove that the convergence rates of $f_t$ to $\hat{Q}^\pi$  under $\|\cdot\|_n$ and $\|\cdot\|_\CalH$ have the same order as their counterpart in Theorem \ref{thm:convergence_KRR}.

    Note that by the selections of decay weight $\alpha$ and step size $\eta$, $\alpha/\eta=\lambda$. Then we can use \eqref{eq:TD_iteration_2} and H\"older inequality to have 
    \begin{equation}
        \label{eq:convergence_Kernel_TD_1}
        \begin{aligned}
             \|f_t-\hat{Q}^\pi\|_n^2=&\frac{1}{n}\left(\Bb_t-\Bb^\pi\right)^\top\BK^2\left(\Bb_t-\Bb^\pi\right)\\
             \leq &\left(\max_{\omega}K(\omega,\omega)\right)^2n\left(\Bb_t-\Bb^\pi\right)^2\\
    =&Cn\left(\left[((1-\alpha)\rm{I}-\eta\BK+\eta\gamma\BC\right]^{t}\left[\Bb_0-\Bb^\pi\right] \right)^2\\
            \leq & Cn \Lambda_2^t (\Bb_0-\Bb^\pi)^2
        \end{aligned}
    \end{equation}
    
    where $\Lambda_2$ is the maximum eigenvalue $\left[((1-\alpha)\rm{I}-\eta\BK+\eta\gamma\BC\right]\top \left[((1-\alpha)\rm{I}-\eta\BK+\eta\gamma\BC\right]$ and $C=(\max_{\omega}K(\omega,\omega))^2$.

    Note that the maximum and minimum eigenvalue of the matrix $\BK-\gamma\BC$ is upper bounded by $(1+\gamma)\max_{\omega}K(\omega,\omega)n$ and lower bounded by $-\gamma \max_{\omega}K(\omega,\omega)n$, respectively. Therefore, if we can select $\eta$ such that
    \begin{equation}
        \label{eq:convergence_Kernel_TD_2}
        \begin{aligned}
        &(1-\alpha -(1+\gamma)\max_{\omega}K(\omega,\omega)n\eta)\leq C_1<1.\\
         &(1-\alpha +\gamma\max_{\omega}K(\omega,\omega)n\eta)\leq C_2<1,
    \end{aligned}
    \end{equation}
    then $\Lambda_2\leq C_3^2<1$ where $C_3=\max\{C_1,C_2\}$. We also need an extra condition so that $f_t$ converges to the estimator $\hat{Q}^\pi$ with correctly selected $\lambda$ as in Theorem \ref{thm:convergence_KRR}, we also have
    \begin{equation}
        \label{eq:convergence_Kernel_TD_3}
        \alpha=\eta\lambda n.
    \end{equation}
    Solving \eqref{eq:convergence_Kernel_TD_2} and \eqref{eq:convergence_Kernel_TD_3} to get
    \begin{equation*}
       0< \eta\leq \frac{1-C_1}{n(1+\gamma)C+\lambda}
    \end{equation*}
    and then by letting $t=\log_{C_3}\left([n(\Bb_0-\Bb^\pi)^2]^{-1}\lambda\right)$, we have
    \begin{align}
    \label{eq:convergence_Kernel_TD_4}
        \|f_t-\hat{Q}^\pi\|_n^2\leq C n (\Bb_0-\Bb_1)^2C_3^t=Cn(\Bb_0-\Bb^\pi)^2 C_3^{\log_{C_3} [n(\Bb_0-\Bb^\pi)]^{-1}\lambda}\leq\CalO(\lambda).
    \end{align}

Similarly, for the convergence in RKHS norm, we have
\begin{equation}
        \label{eq:convergence_Kernel_TD_5}
        \begin{aligned}
             \|f_t-\hat{Q}^\pi\|_\CalH^2=&\left(\Bb_t-\Bb^\pi\right)^\top\BK\left(\Bb_t-\Bb^\pi\right)\\
             \leq &\left(\max_{\omega}K(\omega,\omega)\right)n\left(\Bb_t-\Bb^\pi\right)^2\\
    =&\sqrt{C}n\left(\left[((1-\alpha)\rm{I}-\eta\BK+\eta\gamma\BC\right]^{t}\left[\Bb_0-\Bb^\pi\right] \right)^2\\
            \leq & \sqrt{C}n \Lambda_2^t (\Bb_0-\Bb^\pi)^2.
        \end{aligned}
    \end{equation}
Note that the order of \eqref{eq:convergence_Kernel_TD_5} is the same as that of \eqref{eq:convergence_Kernel_TD_1} except that the constant term is changed from $C$ to $\sqrt{C}$. Therefore, the same iteration number $t$ is enough for the following convergence 
    \begin{equation}
    \label{eq:convergence_Kernel_TD_6}
         \|f_t-\hat{Q}^\pi\|_\CalH^2\leq \CalO(\lambda).
    \end{equation}
    Substitute \eqref{eq:convergence_Kernel_TD_4} into \eqref{eq:convergence_Kernel_TD_0} and \eqref{eq:convergence_Kernel_TD_6} into 
    \eqref{eq:convergence_Kernel_TD_01}, we can have the final results.
\end{proof}

\section{Proof of Theorem \ref{thm:convergence_KRR}}
\label{Sec:pf_convergence_KRR}

    We first slightly modify the error decomposition \eqref{eq:error_decomposition} to get the following  inequality
    \begin{equation}
        \begin{aligned}
        \label{eq:convergence_KRR_1}
            &\frac{1}{n}\sum_{i=1}^n\left(\CalD^{\pi}(\omega^{(i)}_0)^2-\gamma \CalD^{\pi}(\omega^{(i)}_0)\CalD^{\pi}(\omega^{(i)}_1)\right)+\lambda \|\hat{Q}^\pi\|_\CalH^2\\
            =&\frac{1}{n}\sum_{i=1}^n\varepsilon_i\CalD^\pi(\omega_0^{(i)})+\lambda\langle Q^\pi,\hat{Q}^\pi\rangle_\CalH.
        \end{aligned}
    \end{equation}
    By further rewriting \eqref{eq:convergence_KRR_1}, we obtain the oracle  that is essential to the proof:
    \begin{equation}
        \label{eq:convergence_KRR_2}
        \frac{1}{n}\sum_{i=1}^n\left(\CalD^{\pi}(\omega^{(i)}_0)^2-\gamma \CalD^{\pi}(\omega^{(i)}_0)\CalD^{\pi}(\omega^{(i)}_1)\right)+\lambda \|\hat{Q}^\pi\|_\CalH^2 =\underbrace{\frac{1}{n}\sum_{i=1}^n\varepsilon_i\CalD^\pi(\omega_0^{(i)})}_{\text{variance}}+\underbrace{\lambda \langle \hat{Q}^\pi,Q^\pi\rangle_{\CalH}^2}_{\text{bias}}.
    \end{equation}
    For the left-hand side of \eqref{eq:convergence_KRR_2}, we can use standard central limit theorem for Monte Carlo integration to have
    \begin{equation}
        \label{eq:convergence_KRR_3}
    \begin{aligned}
        \gamma\frac{1}{n}\sum_{i=1}^n \CalD^{\pi}(\omega^{(i)}_0)\CalD^{\pi}(\omega^{(i)}_1)\leq &\gamma \big|\E[\CalD^\pi(\omega_0)\CalD^\pi(\omega_1)]\big|+ \CalO_p(n^{-\frac{1}{2}})\left|\E[\CalD^\pi(\omega_0)\CalD^\pi(\omega_1)]\right|\\
        \leq &\left( \gamma+\CalO_p(n^{-\frac{1}{2}}) \right)  \sqrt{\E[|\CalD^\pi(\omega_0)|^2]\E[|\CalD^\pi(\omega_1)|^2]}\\
        \leq & c\left( \gamma+\CalO_p(n^{-\frac{1}{2}}) \right)  {\E[|\CalD^\pi(\omega_0)|^2]}\\
        \leq & c\gamma \frac{1}{n}\sum_{i=1}^n\CalD^{\pi}(\omega^{(i)}_0)^2+\CalO_p(n^{-\frac{1}{2}}){\E[|\CalD^\pi(\omega_0)|^2]},
    \end{aligned}
    \end{equation}
    where the third line of \eqref{eq:convergence_KRR_3} follows from Assumption \ref{assump:dist_mu_gamma} that
\begin{align*}
    \E[|\CalD^\pi(\omega_1)|^2]=&\int \left|\CalD^\pi(s_1,a_1)\right|^2\pi(a_1|s_1)P(s_1|s_0,a_0)\pi(a_0|s_0)\mu(s_0)ds_1da_1ds_0da_0\\
    =& \int \left|\CalD^\pi(s_1,a_1)\right|^2\pi(a_1|s_1)\left[\int P(s_1|s_0,a_0)\pi(a_0|s_0)\mu(s_0)ds_0da_0\right]ds_1da_1\\
    \leq &\int \left|\CalD^\pi(s_1,a_1)\right|^2\pi(a_1|s_1)\left(\max_{s,a}P(s_1|s,a)\right)ds_1da_1\\
    \leq & c\int \left|\CalD^\pi(s_1,a_1)\right|^2\pi(a_1|s_1)\mu(s_1)ds_1da_1\\
    =&c^2\E[|\CalD^\pi(\omega_0)|^2].
\end{align*}

    For the right-hand side of \eqref{eq:convergence_KRR_2}, we use empirical processes \citep{geer2000empirical} to estimate the variance term in \eqref{eq:convergence_KRR_2}. In order to do so, we first show that the Bellman residual is zero-mean sub-Gaussian. From the definition of Bellman operator \eqref{eq:bellman}, the expectation of the Bellman residual is 0, i.e., $\E\varepsilon_i=0$ for all $i=1,\cdots,n$.  From Assumption \eqref{assump:r_pi_P_RKHS} and Lemma \ref{lem:Q_pi_RKHS}, we can know that the Bellman residual is bounded:
    \begin{equation*}
        |r(\omega_0^i)+\gamma Q^{\pi}(\omega_1^{(i)})-Q^{\pi}(\omega_0^{(i)})| \leq \|r\|_{\CalH}+(1-\gamma)\|Q^{\pi}\|_{\CalH}
    \end{equation*}
    where the right-hand side because $\max_\omega f(\omega)=\langle f,K(\cdot,\omega)\rangle_\CalH\leq \|K(\cdot,\omega)\|_\CalH\|f\|_{\CalH}$ for any $f\in\CalH$. So $\varepsilon_i$ is sub-Gaussian.

    For natation simplicity, first define the empirical inner product and empirical norm for any pair of functions $f_1$ and $f_2$ defined on $\{\omega_0^{(i)}\}$:
    \[\langle f_1,f_2\rangle_n=\frac{1}{n}\sum_{i=1}^nf_1(\omega_0^{(i)})f_2(\omega_0^{(i)}),\quad, \|f_1\|_n^2=\langle f_1,f_1\rangle_n.\]
    
    The sub-Gaussianity of Bellman residual $\{\varepsilon_i\}$  allows us apply the following lemma to estimate the variance terms.

\begin{lemma}\label{lem:Modulus-Continuity}

Suppose the RKHS $\CalH$ satisfies Assumption~\ref{assump:entropy}.
Suppose that  $\{\varepsilon_i\}_{i=1}^n$ are Bellman residual as defined in \eqref{eq:bellman_residual}. 
Then,  for all $t$ large enough, 
\begin{equation}\label{eq:Modulus-of-Continuity}
        \sup_{f\in\CalH}\frac{ n^{\frac{1}{2}}\langle \boldsymbol{\varepsilon},f\rangle_n}{{\|f\|_n^{1-\beta}\|f\|^{\beta}_{\CalH} \left|\log\frac{\|f\|_n}{\|f\|_{\CalH}}\right|^{\kappa}}} > t ,
\end{equation}
with probability at most $C_2\exp(-C_1t^2)$ with some positive constants $C_1$ and $C_2$.
\end{lemma}

We can substitute \eqref{eq:convergence_KRR_2} and \eqref{eq:Modulus-of-Continuity} in Lemma \ref{lem:Modulus-Continuity} into the oracle \eqref{eq:convergence_KRR_2}:
\begin{equation}
\label{eq:convergence_KRR_4}
\begin{aligned}
    &(1-c\gamma)\|\CalD^{\pi}\|_n^2 +\lambda \|\hat{Q}^\pi\|_\CalH^2\\
    \leq&\CalO_p\left(n^{-1/2}\right)\|\CalD^\pi\|_n^{1-\beta}\|\CalD^\pi\|_{\CalH}^{\beta}\left|\log \frac{\|\CalD^\pi\|_n}{\|\CalD^\pi\|_{\CalH}}\right|^\kappa +\lambda \langle Q^\pi,\hat{Q}^\pi\rangle_\CalH.
\end{aligned}
\end{equation}

\textbf{Case I $\|\hat{Q}^\pi\|_\CalH \geq \|Q^\pi\|_\CalH$}, \eqref{eq:convergence_KRR_4} yields two subcases, either
\begin{equation}
\label{eq:convergence_KRR_5} 
\begin{aligned}
    &(1-c\gamma))\|\CalD^{\pi}\|_n^2+\lambda\|\hat{Q}^\pi\|_\CalH^2\\
    \leq&\CalO_p\left(n^{-1/2}\right)\|\CalD^\pi\|_n^{1-\beta}\|\hat{Q}^\pi\|_{\CalH}^{\beta}\left||\log{\|\CalD^\pi\|_n}|+|\log{\|\hat{Q}^\pi\|_{\CalH}}|\right|^\kappa
\end{aligned}
\end{equation}
or
\begin{equation}
\label{eq:convergence_KRR_6} 
(1-c\gamma)\|\CalD^{\pi}\|_n^2+\lambda\|\hat{Q}^\pi\|_\CalH^2
    \leq\lambda \langle Q^\pi,\hat{Q}^\pi\rangle_\CalH+\CalO_p(n^{-\frac{1}{2}})\|\hat{Q}^\pi\|_{\CalH}^2.
\end{equation}

Solving \eqref{eq:convergence_KRR_5} yields
\begin{equation}
\label{eq:convergence_KRR_7}
    \begin{aligned}
        &\|\CalD^{\pi}\|_n=({1-c\gamma})^{\frac{\beta}{2}-1}\CalO_p(n^{-\frac{1}{2}})\lambda^{-\beta}(|\log n|+|\log \lambda|)^\kappa,\\
        &\|\hat{Q}^{\pi}\|_\CalH=({1-c\gamma})^{\frac{\beta-1}{2}}\CalO_p(n^{-\frac{1}{2}})\lambda^{-\frac{1+\beta}{2}}(|\log n|+|\log \lambda|)^\kappa.
    \end{aligned}
\end{equation}
Solving \eqref{eq:convergence_KRR_6} yields
    \begin{equation}
\label{eq:convergence_KRR_8}
    \begin{aligned}
        &\|\CalD^{\pi}\|_n\leq \frac{2}{1-c\gamma}\lambda\|Q^\pi\|_\CalH,\\
        &\|\hat{Q}^\pi\|_{\CalH}\leq 2\|Q^\pi\|_{\CalH}.
    \end{aligned}
\end{equation}

\textbf{Case II $\|\hat{Q}^\pi\|_\CalH \leq \|Q^\pi\|_\CalH$}, \eqref{eq:convergence_KRR_4} yields another two subcases, either
\begin{equation}
    \label{eq:convergence_KRR_9}
    (1-c\gamma)\|\CalD^\pi\|_n^2\leq\CalO_p(n^{-1/2})\|\CalD^\pi\|_n^{1-\beta}\|Q^\pi\|_{\CalH}^{\beta},
\end{equation}
or
\begin{equation}
    \label{eq:convergence_KRR_10}
    (1-c\gamma)\|\CalD^\pi\|_n^2\leq 2\lambda \|Q^\pi\|_\CalH^2.
\end{equation}
It follows that either \eqref{eq:convergence_KRR_9} holds or
\begin{equation}
    \label{eq:convergence_KRR_11}
    \|\CalD^\pi\|_n \leq (1-c\gamma)^{\frac{-1}{1+\beta}}\CalO_p(n^{-\frac{1}{2+2\beta}}).
\end{equation}

Now we can summarize \eqref{eq:convergence_KRR_8}, \eqref{eq:convergence_KRR_9}, \eqref{eq:convergence_KRR_10}, and \eqref{eq:convergence_KRR_11}. We can notice that by selecting
\begin{equation}
    \label{eq:convergence_KRR_12}
    \lambda =\CalO((1-c\gamma)^{\frac{\beta}{2+2\beta}}n^{-\frac{1}{2+2\beta}}|\log n|^{\frac{\kappa}{1+\beta}}),
\end{equation}
we can have the final result,
\begin{align}
    &\|\CalD^\pi\|_n\leq \CalO_p\left((1-c\gamma)^{-\frac{2+\beta}{2+2\beta}}n^{-\frac{1}{2+2\beta}}|\log n|^{\frac{\kappa}{1+\beta}}\right)\|Q^\pi\|_\CalH,\label{eq:convergence_KRR_13}\\
    &\|\hat{Q}^\pi\|_{\CalH}\leq \CalO_p(1)\|Q^\pi\|_{\CalH}. \label{eq:convergence_KRR_14}
\end{align}

\subsection{Proof of Lemma \ref{lem:Modulus-Continuity}}
To proof Lemma \ref{lem:Modulus-Continuity}, we first need the following lemma from \cite{geer2000empirical}
\begin{lemma}[Corollary 8.3 in \cite{geer2000empirical}] 
    \label{lemma:geer8_3}
Let $a>0$, $R>0$, and $\CalH$ be a function space.
Suppose that $\sup_{g\in\CalH}\|f\|_{L_\infty}\leq R$ for and  $\boldsymbol{\varepsilon}=\{\varepsilon_i\}_{i=1}^n$ are independent zero-mean sub-Gaussian random variables. 
If there exists some universal positive constant $C$ satisfying
\begin{equation}
    \label{eq:entropy_condition}
    a\geq Cn^{-\frac{1}{2}} \bigg(\int^R_0 H^{\frac{1}{2}}(u,\|\cdot\|_{L_\infty},\CalH)du\vee R\bigg),
\end{equation}
then we have
\begin{equation}
    \label{eq:geer_modulus_continuity}
    \mathbb{P}\left(\sup_{f\in\CalH}\left|\langle \boldsymbol{\varepsilon},f\rangle_n\right|\geq a\right)\leq 2 \exp\biggl(-\tilde{C}\frac{na^2}{R^2}\biggr),
\end{equation}
for some positive $\tilde{C}$.
\end{lemma}
\begin{proof}

Let $\CalB=\{f:\|f\|_{\CalH}\leq 1\}$ be the unit ball in $\CalH$. Recall from Assumption \ref{assump:entropy} that
\[ H(\delta,\|\cdot\|_{L_\infty},\CalB)\leq C\delta^{-2\beta}\left|\log \delta \right|^{2\kappa},\]
for some positive constant $C$, $\kappa$, and $\beta\in[0,1)$. 
Hence, 
\begin{align*}
    \int_0^{R}\left(H(u,\|\cdot\|_{L_\infty},\CalB)\right)^{\frac{1}{2}}d{u}\leq C\int_0^R u^{-\beta }\bigl|\log u\bigr|^{\kappa}d{u}\leq C R^{1-\beta}\bigl|\log R\bigr|^{\kappa}.
\end{align*}

One may readily check that 
for all $R<1$ and $a\geq n^{-1/2}C R^{1-\beta}\bigl|\log R\bigr|^{\kappa}$, the condition \eqref{eq:entropy_condition} is met, 
and thus by \eqref{eq:geer_modulus_continuity}, we have 
\begin{align}
  \pr\biggl(\sup_{f\in\CalH,\|f\|_n\leq R}n^{\frac{1}{2}}|{\langle \boldsymbol{\omega},f\rangle_n}|\geq R^{1-\beta}|\log R|^{\kappa} \biggr)&=\pr\biggl(\sup_{f\in\CalH,\|f\|_n\leq R}\bigl|\langle \boldsymbol{\omega},f\rangle_n\bigr|\geq a\biggr) \nonumber \\
  &\leq 2\exp(-C_1\frac{na^2}{ R^2}),\label{eq:prob-empirical-inner-prod} 
\end{align}
for some positive constant $C_1$.  

Let $h= f/\|f\|_{\CalH}$. Then, the left-hand-side of the inequality \eqref{eq:Modulus-of-Continuity} becomes
\[\sup_{h\in\CalH}\frac{|{\langle \boldsymbol{\varepsilon},h\rangle_n}|}{{\|h\|_n^{1-\beta}}\bigl|\log \|h\|_n\bigr|^{\kappa}}.\]
We then can give the following upper bound for its tail distribution:
\begin{align*}
    \quad{} & \pr\Biggl(\sup_{h\in\CalH}\frac{|{\langle \boldsymbol{\varepsilon},h\rangle_n}|}{{\|h\|_n^{1-\beta}}\bigl|\log \|h\|_n\bigr|^{\kappa}}> t \Biggr)\\
    ={} & \pr\Biggl(\bigcup_{i=0}^{\infty} \biggl\{\sup_{h\in\CalH}\frac{|{\langle \boldsymbol{\varepsilon},h\rangle_n}|}{{\|h\|_n^{1-\beta}}\bigl|\log \|h\|_n\bigr|^{\kappa}}> t \biggr\}\Biggr)\\
    \leq{} & \pr\Biggl(\bigcup_{i=0}^{\infty} \biggl\{\sup_{h\in\CalH,\|h\|_n\in(2^{-i-1}, 2^{-i}]} n^{\frac{1}{2}}|{\langle \Bvaresilon,h\rangle_n}|> t 2^{-{(i+1)}{(1-\beta)}}(i+1)^{\kappa}\biggr\}\Biggr)\\
    \leq{} & \sum_{i=0}^\infty\pr\biggl(\sup_{h\in\CalH, \|h\|_n\leq 2^{-i}}n^{\frac{1}{2}}|{\langle \Bvaresilon,f\rangle_n}|> t2^{-{(i+1)}{(1-\beta)}}i^{\kappa}\biggr)\\
    \leq{} & \sum_{i=0}^{\infty}2\exp(-C_2t^2 2^{\beta i}i^{\kappa})\\
    \leq {} & \sum_{i=1}^\infty 2\bigl(\exp(-C_2t^2)\bigr)^i\\
    \leq{} & C_3 \exp(-C_2 t^2),
\end{align*}
where the fifth line follows from applying \eqref{eq:prob-empirical-inner-prod}  with $R=2^{-i}$ and $a=tn^{-1/2} 2^{-{(i+1)}{(1-\beta)}}i^{\kappa}$. 
\end{proof}

\section{Proof of Corollary \ref{coro:convergence_L2}}
\label{sec:proof_generalization}
\begin{proof}
    The proofs  relies on the following so-called sampling inequality of RKHSs
    \begin{equation}
        \label{eq:sampling_inequality}
        \|f\|^2_{L_2(\sigma^\pi_0)}\leq \|f\|_n^2+\delta_{\CalH,n}\|f\|_{\CalH}^2
    \end{equation}
    where $\delta_{\CalH,n}$ depends on the distribution of data $\omega_0^{(i)}$ and the structure of RKHS $\CalH$.

   \textbf{Case I: $\CalH_\CalS=\{s\}_{s=1}^S$ and $\CalA=\{a\}_{a=1}^A$}
   
   In this case, we can note the the RKHS is a discrete set $\{\omega_j\}_{j=1}^{M}$ where $M=SA$, $\sigma^\pi_0(\omega_j)=\frac{W_j}{M}$ is the weight on element $\omega_j$ (according to our assumption,  $c<W_j<C$), and data $\{\omega_0^{(i)}\}$ are i.i.d. samples from $\{\omega_j\}_{j=1}^M$ following $\sigma^\pi_0$.

   Because the kernel $K(\omega,\omega')=\delta_{s=s'}\delta_{a=a'}=\delta_{\omega=\omega'}$, we can  derive that the RKHS is equipped with the inner product
   \[\langle f,g\rangle_\CalH=\sum_{j=1}^M f_jg_j,\quad f,g\in\CalH\]
where $f_j$ and $g_j$ are the values of $f$ and $g$ on $\omega_j$, respectively. Also, the $L_2$ norm and empirical norm are
\begin{align*}
    \|f\|_{L_2(\sigma^\pi_0)}^2=\sum_{j=1}^Mf_j^2W_j,\quad \|f\|_n^2=\frac{1}{n}\sum_{i=1}^nf^2(\omega_i^0).
\end{align*}
Let $I_{i,j}$ denote the indicator function that $\omega_0^{(i)}=\omega_j$, then we can rewrite the empirical norm as follows:
\begin{equation}
    \label{eq:convergence_L2_discrete_1}
    \begin{aligned}
    \|f\|_n^2=&\frac{1}{n}\sum_{i=1}^nf^2(\omega_i^0)=\sum_{j=1}^Mf_j^2\frac{\sum_{i=1}^nI_{i,j}}{n}\\
    =& \sum_{j=1}^Mf_j^2\frac{W_j}{M}+\sum_{j=1}^M f_j^2\left(\frac{\sum_{i=1}^nI_{i,j}}{n}-\frac{W_j}{M}\right)\\
    =&\|f\|^2_{L_2(\sigma^\pi_0)}+\sum_{j=1}^M f_j^2\left(\frac{Z_j}{n}-\frac{W_j}{M}\right)
\end{aligned}
\end{equation}
where $[Z_j]_{j=1}^M$ is Multinomial distributed with $n$ trials, $M$ mutually exclusive events, and probability vector $[W_j/m]_{j=1}^M$.

By properties of Multinomial distribution, the last term in \eqref{eq:convergence_L2_discrete_1} has expecatation
\begin{equation}
    \label{eq:convergence_L2_discrete_2}
    \E\sum_{j=1}^M f_j^2\left(\frac{Z_j}{n}-\frac{W_j}{M}\right)=\sum_{j=1}^M f_j^2 \left(\frac{nW_j/m}{n}-\frac{W_j}{M}\right)=0,
\end{equation}
and variance 
\begin{equation}
    \label{eq:convergence_L2_discrete_3}
    \begin{aligned}
        \text{Var}\left(\sum_{j=1}^M f_j^2\left(\frac{Z_j}{n}-\frac{W_j}{M}\right)\right)=&\frac{1}{n^2}\sum_{j,l=1}^mf_j^2f_l^2\text{Cov}\left(Z_j,Z_l\right)\\
        =&\frac{1}{n^2}\left(\sum_{j=1}^mf_j^2\frac{nW_j}{M}(1-\frac{W_j}{M})-\sum_{j\neq l}f_j^2f_l^2\frac{nW_jW_l}{M^2}\right)\\
        \leq &\CalO(\frac{1}{n^{1+\nu}}\|f\|_{\CalH}^2)
    \end{aligned}
\end{equation}
where the last line is from our assumption that $M\geq n^\nu$ for some $\nu\in(0,1)$.

Substitute \eqref{eq:convergence_L2_discrete_2} and \eqref{eq:convergence_L2_discrete_3} into \eqref{eq:convergence_L2_discrete_1}, it is straightforward to derive the sampling inequality
\begin{equation}
    \label{eq:convergence_L2_discrete_4}
    \|f\|_n^2\leq \|f\|_{L_2(\sigma^\pi_0)}^2+ \CalO_p\left(\frac{1}{n^{(1+\nu)/2}}\right)\|f\|_{\CalH}^2.
\end{equation}

It is straightforward to see that the covering number of $\CalH$ under $L_\infty$ is directly the covering number of $\CalH$ is $M/\delta$ so its entropy is
\[H(\delta,\|\cdot\|_{L_\infty},\CalH)\leq \nu\log n+|\log \delta|.\]
Substitute this result into  the sampling inequality, we can have the final result:
\begin{align*}
    \|f_t-Q^\pi\|_{L_2(\sigma^\pi_0)}\leq & \|f_t-Q^\pi\|_{n}+\CalO_p\left(\frac{1}{n^{(1+\nu)/4}}\right)\|f_t-Q^\pi\|_{\CalH}\\
    \leq & \CalO_p\left((1-c\gamma)^{-1}n^{-\frac{1}{2}}|\log n|^{1/2}\right)+\CalO_p\left(\frac{1}{n^{(1+\nu)/4}}\right)\|f_t-Q^\pi\|_{\CalH}\\
    \leq& \CalO_p\left((1-c\gamma)^{-1}n^{-\frac{1}{2}}|\log n|^{1/2}\right)+\CalO_p\left(\frac{1}{n^{(1+\nu)/4}}\right)\|Q^\pi\|_{\CalH}
\end{align*}
where the second line is from the empirical error in Theorem \ref{thm:convergence_Kernel_TD} with $\beta=0$ and $\kappa=1/2$ and the last line is from the RKHS norm of $f_t$ in Theorem \ref{thm:convergence_Kernel_TD}.

 \textbf{Case II: $\CalH_\CalS$ is a Sobolev space  and $\CalA=\{a\}_{a=1}^A$}

 For spaces of continuous functions, we first need to define a concept called filled distance of data set so that we can apply the sampling inequality.

 \begin{definition}
     \label{def:fill_distance}
     Given $n$ points $\BX=\{\Bx_i\}_{i=1}^n$ on the set $\CalX$ equipped with norm $\|\cdot\|$, the fill distance of the set $\BX$ is
     \[q_\BX=\max_{\Bx\in\CalX}\min_{\Bx_i\in\BX}\|\Bx_i-\Bx\|_2.\]
 \end{definition}
Using the concept of fill distance, we have the following Lemma from \cite{tuo2020improved,rieger2008sampling,utreras1988convergence}  
\begin{proposition}[Proposition 2.6 in \cite{tuo2020improved}]
    \label{prop:samplimg_inequality_fill_distance}
    Let $\CalH$ be a Sobolev space for function defined on a compact and convex set $\CalX$ with smoothness parameter $m$. Let $\BX=\{\Bx_i\}_{i=1}^n$ be a set of scattered point on $\CalX$ with fill distance $q_\BX$ and $\|\cdot\|_n$ is the empirical norm induced by $\BX$. Then there exists a constant $C$ (depending only on $m$ and $\CalX$) such that for any $f\in\CalH$:
    \[\|f\|^2_{L_2(\CalX)}\leq C\left(\|f\|^2_{n}+ q_\BX^{2m}\|f\|_{\CalH}^2\right). \]
\end{proposition}
We also need the following proposition for the distribution of fill distance:
\begin{proposition}[Proposition 3 in \cite{krieg2024random}]
\label{prop:fill_dist_prob}
    Let $\{\Bx_i\sim\sigma\}_{i=1}^{n}$ be i.i.d. distributed on a compact and convex set $\CalX$. Let $\Phi(\varepsilon)=\varepsilon^d$. If there exists a positive number $\varepsilon_0$ such that $\sigma(B_\varepsilon(\Bx))\geq \Phi(\varepsilon)$ for all $\varepsilon\leq \varepsilon_0$ and $\Bx\in B$ where $B_\varepsilon(\Bx)$ is a sphere centered at $\Bx$ with radius $\varepsilon$. Then there exist positive constants $c_1$, $c_2$, $c_3$, and $b_0$ such that for any $b>b_0$ we have
    \[\Pr\left(\max_{\Bx\in \CalS}\inf_i\|\Bx-\Bx_i\|\geq c_1\Phi^{-1}(\frac{b\log n}{n})\right)\leq c_2n^{1-c_3 b}.\]
\end{proposition}
From Proposition \ref{prop:fill_dist_prob}, we can derive that the filled distance on $\CalX$ with intrinsic dimension $d$ is 
\begin{equation}
 \label{eq:convergence_L2_Sobolev_1}
    q_{\BX}=\CalO_p(n^{-1/d}\log n).
\end{equation}

Lastly, from \cite{geer2000empirical,birman1967piecewise}, we have the entropy of Sobolev spaces $\CalH$ with smoothness $m$ defined on compact $d$-dimensional domain is
\begin{equation}
    H(\delta,\|\cdot\|_{L_\infty},\CalH)\leq C_1 \delta^{-\frac{d}{m}}
\end{equation}
where $C_1$ is some constant independent of $\delta$.

Now we are ready to prove the sampling inequality in the form \eqref{eq:sampling_inequality} for our RKHS. According to our assumption, $\CalH_\CalS$ is a Sobolev space with smoothness parameter $m$ embedded on a domain with intrinsic dimension $d$ and $\CalH_\CalA$ is a discrete set. So it is straightforward to derive that the overall RKHS $\CalH$ is a vector-valued function $(f_1,\cdots,f_A)$ where each $f_a\in\CalH_\CalS$, and $\CalH$ is define as
\[\CalH=\left\{f=(f_1,\cdots,f_A): \|f\|_{\CalH}^2=\sum_{a=1}^A\|f_a\|^2_{\CalH_\CalS}<\infty\right\}.\]

Because $f_t,Q^\pi\in\CalH$, we then have
\begin{equation}
    \label{eq:convergence_L2_Sobolev_2}
    \begin{aligned}
    \|f_t-Q^\pi\|_{L_2(\sigma^\pi_0)}^2=&\sum_{a=1}^A \|f_{a,t}-Q^\pi_a\|_{L_2(\sigma^\pi_0)}^2\\
    \leq &\sum_{a=1}^A C\left(\|f_{a,t}-Q^\pi_a\|_{n}^2+q_{\Bomega_0}^{2m}\|f_{a,t}-Q^\pi_a\|_{\CalH_\CalS}^2\right)\\
    \leq &\sum_{a=1}^AC\left(\|f_{a,t}-Q^\pi_a\|_{n}^2+\CalO_p(n^{-\frac{2m}{d}}|\log n|^{2m})\|f_{a,t}-Q^\pi_a\|_{\CalH_\CalS}^2\right)\\
    \leq &\CalO_p\left((1-c\gamma)^{-\frac{4m+d}{2m+2d}}n^{-\frac{2m}{2m+d}}\right)+\CalO_p(n^{-\frac{2m}{d}}|\log n|^{2m})\|Q^\pi\|_{\CalH}^2\\
    \leq & \CalO_p\left((1-c\gamma)^{-\frac{4m+d}{2m+2d}}n^{-\frac{2m}{2m+d}}\right)
\end{aligned}
\end{equation}
where the second line is from Proposition \ref{prop:samplimg_inequality_fill_distance}, the third line is from the distribution of fill distance \eqref{eq:convergence_L2_Sobolev_1}, the fourth line is from Theorem \ref{prop:KRR_TD_representer_thm} with $\beta=\frac{d}{2m}$ and $\kappa=0$.

\textbf{Case III: $\CalH_\CalS$ is a NTK and $\CalA=\{a\}_{a=1}^A$}

According to \cite{chen2020deep}, the NTK $N$ associated to a two-layer neural network on $\mathbb{S}^{d-1}$ is equivalent to the Laplace kernel $e^{-\|s-s'\|}$ on $\mathbb{S}^{d-1}$. The Laplace kernel is a Mat\'ern kernel with smoothness parameter $\nu=1/2$. Also, according to Corollary 1 in \cite{tuo2016theoretical}, RKHS induced by Mat\'ern kernel with smoothness parameter $\nu$ is equivalent to (fractional) Sobolev space with smoothness $m=\nu+d/2$. 

Therefore, we can conclude that $\CalH$ is equivalent to a Sobolev space with smoothness $(d+1)/2$ embedded on $\mathbb{S}^{d-1}$, which has intrinsic dimension $d-1$. So Case III is a special case of Case II, i.e. Sobolev space with smoothness $m=(d+1)/2$ embedded on $d-1$ dimension. By substituting the value of $m$ and $d-1$ into \eqref{eq:convergence_Sobolev}, we can have the final result. 

 \textbf{Case IV: $\CalH_\CalS$ and $\CalH_\CalA$ are Gaussian RKHSs}

Note that $K(\omega,\omega')=e^{-\|\omega-\omega'\|^2}$ is a Gaussian kernel on the hypercube $[0,1]^d$ with $d=d_s+d_a$. So $\CalH$ is a Gaussian RKHS on $[0,1]^d$.

From Chapter 4 in \cite{rieger2008sampling}, we have the sampling inequality:
\begin{equation}
    \label{eq:convergence_L2_Gaussian_1}
    \|f\|_{L_2(\sigma^\pi_0)}\leq e^{c\log(cq_{\Bomega_0})/\sqrt{q_{\Bomega_0}}}\|f\|_{\CalH}+c\|f\|_n,\quad \forall f\in\CalH
\end{equation}
for some generic constant $c>0$. Substitute the distribution of fill distance in \eqref{eq:convergence_L2_Sobolev_1} into \eqref{eq:convergence_L2_Gaussian_1}, we can have
\begin{equation}
    \label{eq:convergence_L2_Gaussian_2}
    \|f\|_{L_2(\sigma^\pi_0)}\leq c\|f\|_n+\CalO_p(e^{-n^{1/\sqrt{d}}})\|f\|_{\CalH},\quad \forall f\in\CalH.
\end{equation}
From Theorem 3 in \cite{kuhn2011covering}, the entropy of $\CalH$ is
\begin{equation}
    \label{eq:convergence_L2_Gaussian_3}
    H(\delta,\|\cdot\|_\infty,\CalB)\leq C_1 |\log \delta|^{d+1}
\end{equation}
where $C_1$ is some generic constant.

So we can have
\begin{align*}
    \|f_t-Q^\pi\|_{L_2(\sigma^\pi0)}\leq &c\|f_t-Q^\pi\|_n+\CalO_p(e^{-n^{1/\sqrt{d}}})\|f_t-Q^\pi\|_{\CalH}\\
    \leq & \CalO_p((1-c\gamma)n^{-\frac{1}{2}}|\log n|^{(d+1)/2})+\CalO_p(e^{-n^{1/\sqrt{d}}})\|Q^\pi\|_{\CalH}\\
    \leq & \CalO_p((1-c\gamma)n^{-\frac{1}{2}}|\log n|^{(d+1)/2})
\end{align*}
where the first line is from the sampling inequality \eqref{eq:convergence_L2_Gaussian_2} and the second line is from entropy \eqref{eq:convergence_L2_Gaussian_3} and Theorem \ref{thm:convergence_Kernel_TD} with $\beta=0$ and $\kappa=(d+1)/2$.

\end{proof}

\section{Proof of Lemma \ref{lem:NPG_optimize}}
\begin{proof}
   We generalize the proof for Proposition 3.1 in \cite{liu2019neural}, following its main lines of reasoning while incorporating additional techniques of RKHS. For notation simplicity, denote the $L_2$ inner product of on $\CalA$ as $\langle f,g\rangle_\CalA=\int_\CalA f(a)g(a)da$.

   The Lagrangian of \eqref{eq:NPG_optimize} takes the form
   \begin{equation}
   \label{eq:NPG_optimize_1}
       \E_n\left[\Delta_k\langle f^{(k)}(s,\cdot )\pi(\cdot|s)\rangle_\CalA- \text{KL}\left(\pi(\cdot|s)\|\pi^k(\cdot|s)\right)\right]+\E_n\left[\left(\int_\CalA\pi(a|s)-1\right)\Lambda(s)\right]
   \end{equation}
   where $\Lambda(s)$ is the $L_2$ Lagrange multiplier defined on $\{s^{(i)}_0\}$.  By taking the functional derivative of \eqref{eq:NPG_optimize_1} with respective to $\pi$, we can have
   \begin{equation}
   \label{eq:NPG_optimize_2}
       \Delta_k f^{(k)}(s,a)+F(s,a)-\left(\log \pi(a|s)+1+\log Z(s)\right)+\Lambda(s)=0
   \end{equation}
   where $Z(s)=\int_\CalA e^{F(s,a)}da$.  Then \eqref{eq:NPG_optimize_2} should holds on any $s\in\{s^{(i)}_0\}$ and $a\in\CalA$ such that $|F(s,a|$ and $\log Z(s)$ is bounded. According to our assumption that $1>\pi^{(k)}>0$, any  $(s,a)\in\{s^{(i)}_0\}\times\CalA$ satisfies this condition. Then we can solve \eqref{eq:NPG_optimize_2}, which yields
   \begin{equation}
       \label{eq:NPG_optimize_3}
           \pi(a|s)=\pi^{k}(a|s)\exp\left\{\Delta_k f^{(k)}(s,a)+\Lambda(s)-1-\log Z(s)\right\}.
   \end{equation}
   In \eqref{eq:NPG_optimize_3}, we can note that $\pi^k(a|s)$ must be proportional to $\pi^{k}(a|s)\exp\{\Delta_k f^{(k)}(s,a)\}$ because it is a density on $a$.
\end{proof}

\section{Proof of Theorem \ref{thm:one_step_improvement}}
\begin{proof}
    From direct calculations, we have for any $s\in\CalS$
       \begin{align}
            &\text{KL}\left(\pi^*(\cdot|s)||\pi^{k+1}(\cdot|s)\right)-\text{KL}\left(\pi^*(\cdot|s)||\pi^{k}(\cdot|s)\right)\nonumber\\
        =&\langle \log \frac{\pi^{k}}{\pi^{k+1}},\pi^*-\pi^{k+1}\rangle_{\CalA}-\langle \log \frac{\pi^{k+1}}{\pi^{k}},\pi^{k+1}\rangle_{\CalA}\nonumber\\
        =&\Delta_k\langle f^{(k)},\pi^*-\pi^{k}\rangle_{\CalA}-\Delta_k\langle f^{(k)},\pi^{k+1}-\pi^{k}\rangle_{\CalA}-\langle \log \frac{\pi^{k+1}}{\pi^{k}},\pi^{k+1}\rangle_{\CalA}\label{eq:one_step_improvement_1}\\
        =&\Delta_k\langle Q^{(k)},\pi^*-\pi^{k}\rangle_{\CalA}-\Delta_k\langle f^{(k)}-Q^{(k)},\pi^{*}-\pi^{k+1}\rangle_{\CalA}\nonumber\\
        &-\Delta_k\langle Q^{(k)},\pi^{*}-\pi^{k+1}\rangle_{\CalA}-\langle \log \frac{\pi^{k+1}}{\pi^{k}},\pi^{k+1}\rangle_{\CalA}\nonumber\\
        \leq &\Delta_k\langle Q^{(k)},\pi^*-\pi^{k}\rangle_{\CalA}-\Delta_k\langle f^{(k)}-Q^{(k)},\pi^{*}-\pi^{k+1}\rangle_{\CalA}\label{eq:one_step_improvement_2}\\
        &-\int_\CalA \Delta_k Q^{(k)}(s,a)\left(\pi^{*}-\pi^{k+1}\right)da-\frac{1}{2}\left(\int_\CalA\left|\pi^{k+1}(a|s)-\pi^k(a|s)\right|da\right)^2\nonumber\\
        \leq  &\Delta_k\langle Q^{(k)},\pi^*-\pi^{k}\rangle_{\CalA}-\Delta_k\langle f^{(k)}-Q^{(k)},\pi^{*}-\pi^{k+1}\rangle_{\CalA}\nonumber\\
        &+\|\Delta_k Q^{(k)}(s,)\|_{L_\infty(\CalA)}\int_\CalA \left|\pi^{*}-\pi^{k+1}\right|da-\frac{1}{2}\left(\int_\CalA\left|\pi^{k+1}(a|s)-\pi^k(a|s)\right|da\right)^2\nonumber\\
         \leq  &{\Delta_k\langle Q^{(k)},\pi^*-\pi^{k}\rangle_{\CalA}}-{\Delta_k\langle f^{(k)}-Q^{(k)},\pi^{*}-\pi^{k+1}\rangle_{\CalA}}+\Delta^2_k\frac{\|r\|_{L_\infty}}{1-\gamma}\label{eq:one_step_improvement_3}
       \end{align}
    where \eqref{eq:one_step_improvement_1} is from the NPG update rule \eqref{eq:NPG}, \eqref{eq:one_step_improvement_2} is from Pinsker inequality, \eqref{eq:one_step_improvement_3} is from the fact that $xy-y^2\leq x^2$ and $Q^{(k)}\leq \max_{s,a}|r(s,a)|/(1-\gamma)$.

    Taking expectation $\E_{S\sim\nu^*}$ on \eqref{eq:one_step_improvement_3}, we can have:
        \begin{align}
            &\E_{S\sim\nu^*}\left[\text{KL}\left(\pi^*(\cdot|s)||\pi^{k+1}(\cdot|s)\right)-\text{KL}\left(\pi^*(\cdot|S)||\pi^{k+1}(\cdot|S)\right)\right]\nonumber\\
        \leq &{\E_{S\sim\nu^*}\Delta_k\langle Q^{(k)}(\cdot,S),\pi^*(\cdot|S)-\pi^{k}\pi^*(\cdot|S)\rangle_{\CalA}}\label{eq:one_step_improvement_4}\\
        &-\E_{S\sim\nu^*}{\Delta_k\langle f^{(k)}(\cdot,S)-Q^{(k)}(\cdot,S),\pi^{*}(\cdot|S)-\pi^{k+1}(\cdot|S)\rangle_{\CalA}}\label{eq:one_step_improvement_5}\\
        &+\Delta^2_k\frac{\|r\|_{L_\infty}}{1-\gamma}\nonumber
        \end{align}
    For  \eqref{eq:one_step_improvement_4}, we can use the performance difference lemma in \cite{kakade2002approximately} (equivalent to \cite{liu2019neural} Lemma 5.1):
    \begin{lemma}[Performance Difference Lemma]
    \label{lem:performance_difference_lemma}
    For $\CalR[\pi]$ defined in \eqref{eq:performance_difference}, we have
    \begin{equation}
        \label{eq:performance_difference_lemma}
        \CalR[\pi]-\CalR[\pi^*]=\E_{S\sim \nu^*}[\langle Q^\pi(S,\cdot),\pi(\cdot|S)-\pi^*(\cdot|S)\rangle_\CalA].
    \end{equation}
    \end{lemma}
    From \eqref{eq:performance_difference_lemma}, we can directly have for \eqref{eq:one_step_improvement_4}
    \begin{equation}
        {\E_{S\sim\nu^*}\Delta_k\langle Q^{(k)}(\cdot,S),\pi^*(\cdot|S)-\pi^{k}\pi^*(\cdot|S)\rangle_{\CalA}}=\CalR[\pi^*]-\CalR[\pi^k].\label{eq:one_step_improvement_6}
    \end{equation}

    For \eqref{eq:one_step_improvement_5},
    \begin{equation*}
        \begin{aligned}
           & \left|\E_{S\sim\nu^*}{\Delta_k\langle f^{(k)}(\cdot,S)-Q^{(k)}(\cdot,S),\pi^{*}(\cdot|S)-\pi^{k+1}(\cdot|S)\rangle_{\CalA}}\right|\\
           \leq & \Delta_k \|f^{(k)}-Q^{k}\|_{L_\infty}\left(\int \pi^*(a|s)\nu^*(s)dads+\int \pi^{k+1}(a|s)\nu^*(s)dads\right)\\
           =&2\Delta_k \|f^{(k)}-Q^{k}\|_{L_\infty}
        \end{aligned}
    \end{equation*}

    Substitute the above upper bounds for \eqref{eq:one_step_improvement_4} and \eqref{eq:one_step_improvement_5} into \eqref{eq:one_step_improvement_3} and rearrange the terms, we have
    \begin{equation}
    \label{eq:one_step_improvement_7}
        \begin{aligned}
           & \Delta_k\left(\CalR[\pi^k]-\CalR[\pi^*]\right)+\E_{S\sim \nu^*}\left[\text{KL}\left(\pi^{*}(\cdot|S)||\pi^{k+1}(\cdot|S)\right)-\text{KL}\left(\pi^{*}(\cdot|S)||\pi^{k}(\cdot|S)\right)\right]\\
      \leq & 2\Delta_k\|f^{(k)}-Q^{(k)}\|_{L_\infty} +\Delta^2_k(1-\gamma)^{-1}R^{(k)}.\\
        \end{aligned}
    \end{equation}
    By summing over all the $k$ in \eqref{eq:one_step_improvement_7}, we can have the final result.
\end{proof}

\section{Proof of Corollary \ref{coro:NGP_convergence}}
\begin{proof}
    The proof mainly relies on the following $L_2$ and $L_\infty$ norm embedding inequality
    \begin{equation}
        \label{eq:interpolation_inequality}
        \|f\|_{L_\infty}\leq C\|f\|_{L_2(\mu)}^{\varphi}\|f\|_{\CalH}^{1-\varphi}
    \end{equation}
    for some constants $C$ and $0\leq \varphi\leq 1$ independent of $f$ and $\mu$ is any measure equivalent to the uniform measure on $\CalS\times\CalA$. Based on Assumption \ref{assump:distribution_bounded} that $\nu^*$ is lower bounded on $\CalS\times\CalA$, we can use \eqref{eq:interpolation_inequality} and Theorem \ref{thm:convergence_Kernel_TD} to determine the parameters in Algorithm \ref{alg:PPO_RKHS} such that the $L_\infty$ error satisfies
    \begin{equation}
        \label{eq:NPG_convergence_1}
        \|f^{(k)}-Q^{(k)}\|_{L_\infty}\leq \frac{1}{\sqrt{k}}
    \end{equation}
    Then by substitute \eqref{eq:NPG_convergence_1} and $\Delta_k=1/\sqrt{k}$ into \eqref{eq:one_step_improvement}, we can have the target result. We prove the cases one by one

    \textbf{Tabular:} On the discrete set,  it is obvious that 
    \begin{equation}
    \label{eq:NPG_convergence_2}
         \|f\|_{L_\infty}=\max_{\omega\in\CalS\times\CalA}|f(\omega)\leq \sqrt{\sum_{\omega\in\CalS\times\CalA}f(\omega)^2}\leq \frac{\|f\|_{L_2(\nu^*)}}{\min_{\omega}\nu^*}.
    \end{equation}
    By letting $n^{(k)}=\CalO(\frac{k\|\pi^k\|^2_\CalH}{(1-c\gamma)^2}\log \frac{k\|\pi^k\|_\CalH}{(1-c\gamma)}$ and $\beta=0$, $\kappa=1/2$ for the tabular case, we can derive from Theorem \ref{thm:convergence_KRR} that the required $\lambda^{(k)}$ for the optimal error rate is
    \begin{equation*}
        \lambda^{(k)}=\CalO\left(n^{-\frac{1}{2}}|\log n|^{1}\right)=\CalO\left(\frac{(1-c\gamma)}{\sqrt{k}\|\pi^k\|_\CalH}{|\log \frac{k\|\pi^k\|_\CalH}{(1-c\gamma)}|^{\frac{1}{2}}}\right).
    \end{equation*}
    We then  substitute \eqref{eq:NPG_convergence_2} into \eqref{eq:convergence_tabular}:
    \begin{equation}
    \label{eq:NPG_convergence_2_5}
        \begin{aligned}
        \|f^{(k)}-Q^{(k)}\|_{L_\infty}\leq &\CalO_p\left((1-c\gamma)^{-1}n^{-\frac{1}{2}}|\log n|^{1/2}\right)\|Q^{(k)}\|_{\CalH}+\CalO_p\left(\frac{1}{n^{(1+\nu)/4}}\right)\|Q^{(k)}\|_{\CalH}\\
        \leq & \CalO_p\left(\frac{1}{\sqrt{k}}\right)+\CalO_p\left(\frac{1}{n^{(1+\nu)/4}}\right)\|Q^{(k)}\|_{\CalH}
    \end{aligned}
    \end{equation}
    where the second line is from Lemma \ref{lem:Q_pi_RKHS}. 

From \eqref{eq:NPG_convergence_2_5}, we can note that we also need to make sure that the second term on the right hand side is in the order $k^{-1/2}$. So we may need more samples so that $\CalO_p\left(\frac{1}{n^{(1+\nu)/4}}\right)\|Q^{(k)}\|_{\CalH}\leq \CalO(k^{-1/2})$. By solving this equation, we can have the target result.

\textbf{Sobolev with intrinsic dimension $d$:} According to the Gagliardo-Nirenberg inequality, when $\CalH$ is a Sobolev space with smoothness parameter $m$ and dimension $d$, we have
\begin{equation}
    \label{eq:NPG_convergence_3}\|f\|_{L_\infty}\leq C\|f\|_{L_2(\nu^*)}^{\frac{2m-d}{2m}}\|f\|_{\CalH}^{\frac{d}{2m}}.
\end{equation}
By letting $n^{(k)}=\CalO\left(\frac{\|\pi^k\|_\CalH^{\frac{2(2m+d)}{(2m-d)}}k^{\frac{2m+d}{2m-d}}}{(1-c\gamma)^{\frac{2m+d/2}{m}}}\right)$ and $\beta=\frac{d}{2m}$, $\kappa=0$ for the Sobolev case, we can derive from Theorem \ref{thm:convergence_KRR} that the required $\lambda^{(k)}$ for the optimal error rate is
\begin{align*}
    \lambda =\CalO((1-c\gamma)^{\frac{d/2}{2m+d}}n^{-\frac{m}{2m+d}}|)=\CalO\left(\frac{(1-c\gamma)}{\|\pi\|_{\CalH}^{\frac{2m}{2m-d}}k^{\frac{m}{2m-d}}}\right)
\end{align*}
Similar to the tabular case, we then  substitute \eqref{eq:NPG_convergence_3} into \eqref{eq:convergence_Sobolev}:
\begin{align*}
    \|f^{(k)}-Q^{(k)}\|_{L_\infty}\leq \CalO_p\left(\left((1-c\gamma)^{-\frac{2m+d/2}{2m+d}}n^{-\frac{m}{2m+d}}\right)^{\frac{2m-d}{2m}}\|\pi^{(k)}\|_{\CalH}\right)\leq\CalO_p\left(\frac{1}{\sqrt{k}}\right).
\end{align*}

\textbf{NTK:} As we have shown in Section \ref{sec:proof_generalization} Case III that 
the NTK $N$ associated to a two-layer neural network on $\mathbb{S}^{d-1}$ is equivalent to the (fractional) Sobolev space with smoothness $m=1/2+d/2$ with intrinsic dimension $d-1$. Substitute the smoothness $(1+d)/2$ and intrinsic dimension $d-1$ into the Sobolev case, we can have the result.

\textbf{Gaussian:} we use Lemma E.4 in \cite{ding2024random} for the interpolation inequality in Gaussian RKHS. It states that for any $f\in\CalH$, there exists a universal constant $C>0$ such that
\begin{equation}
    \label{eq:NPG_convergence_4}
    \|f\|_{L_\infty}\leq C\|f\|_{L_2(\nu^*)}^{1-\epsilon}\|f\|_{\CalH}^{\epsilon}
\end{equation}
for any $\epsilon>0$.

By letting $n^{(k)}=\CalO\left(\frac{\|\pi^k\|^{\frac{2}{1-\epsilon}}k^{\frac{1}{1-\epsilon}}}{(1-c\gamma)^2}\log \frac{\|\pi^k\|_\CalH k}{1-c\gamma}\right)$ and $\beta=0$, $\kappa=(d+1)/2$ for the Gaussian case, we can derive from Theorem \ref{thm:convergence_KRR} that the required $\lambda^{(k)}$ for the optimal error rate is
\begin{align*}
    \lambda =\CalO(n^{-\frac{1}{2}}|\log|^{\frac{d+1}{2}})=\CalO\left(\frac{(1-c\gamma)}{\|\pi\|_{\CalH}^{\frac{1}{1-\epsilon}}\sqrt{k}^{\frac{1}{1-\epsilon}}}\right),\quad \forall \epsilon\in(0,1).
\end{align*}
Similar to the tabular case, we then  substitute \eqref{eq:NPG_convergence_4} into \eqref{eq:convergence_Gaussian}:
\begin{align*}
    \|f^{(k)}-Q^{(k)}\|_{L_\infty}\leq \CalO_p\left(\left((1-c\gamma)^{-1}n^{-\frac{1}{2}}|\log n|^{\frac{d+1}{2}}\right)^{\frac{1}{1-\epsilon}}\|\pi^{(k)}\|_{\CalH}\right)\leq\CalO_p\left(\frac{1}{\sqrt{k}}\right).
\end{align*}
\end{proof}

\section{Experiment Supplements}

\subsection{RL environments}
\label{app:env}
\textit{CartPole-v1 \citep{barto2012neuronlike}}: A classic control task with 4-dimensional continuous state space (cart position, cart velocity, pole angle, pole angular velocity) and 2 discrete actions (push left/right). The goal is to balance a pole on a cart, with an optimal reward of 500.

\textit{Acrobot-v1 \citep{sutton1998reinforcement}}: An underactuated robotics task with 6-dimensional continuous state space (cosine/sine of both joint angles, angular velocities) and 3 discrete actions (torque values $\{-1, 0, +1\}$). The objective is to swing the free end above a target height, with an optimal reward of $-100$.

\subsection{Hyperparameter Configuration}
\label{app:hyper}
Table \ref{tab:hyper_1} shows the hyperparameter setting of convergence analysis experiments in Section \ref{sec:converge}. The only difference in the hyperparameter is the learning rate.


\begin{table}[H]
\centering
\caption{Hyperparameters of convergence analysis in Figure \ref{fig:convergence-cartpole-acrobot}.}
\label{tab:hyper_1}
\begin{tabular}{lcc}
\toprule
\textbf{Hyperparameter} & \textit{CartPole-v1} & \textit{Acrobot-v1} \\
\midrule
Optimizer & \multicolumn{2}{c}{Adam} \\
Learning rate \(\eta\) & \(1\times 10^{-3}\) & \(2.5\times 10^{-4}\) \\
Network & \multicolumn{2}{c}{2-layer MLP} \\
Hidden width \(h_{\mathrm{dim}}\) & \multicolumn{2}{c}{64} \\
Activation & \multicolumn{2}{c}{ReLU} \\
Batch size & \multicolumn{2}{c}{32} \\
PPO Epochs ($T$) & \multicolumn{2}{c}{4} \\
TD Error Loss Coefficient& \multicolumn{2}{c}{0.5} \\
\bottomrule
\end{tabular}
\end{table}

Table~\ref{tab:hyperparameters} presents the hyperparameter configuration used in our comparative study (Section \ref{sec:ppo_comp}) between classic PPO with GAE and our novel NPG implementation with TD learning on CartPole-v1.

\begin{table}[H]
\centering
\caption{Hyperparameter Configuration for CartPole-v1 Comparison Study in Figure \ref{fig:comparison}.}
\label{tab:hyperparameters}
\begin{tabular}{@{}lcc>{\raggedright\arraybackslash}p{5.5cm}@{}}
\toprule
\textbf{Parameter} & \textbf{Value} & \textbf{Unit} & \textbf{Description} \\
\midrule
\multicolumn{4}{l}{\textit{Environment Configuration}} \\
\midrule
Training Episodes & 5,000 & episodes & Total training episodes \\
Update Frequency & 128 & steps & Steps per PPO update \\
\midrule
\multicolumn{4}{l}{\textit{Learning Parameters}} \\
\midrule
Learning Rate & $1 \times 10^{-3}$ & - & Adam optimizer rate \\
Discount Factor ($\gamma$) & 0.99 & - & Future reward discounting \\
GAE Lambda ($\lambda$) & 0.95 & - & GAE bias-variance trade-off \\
\midrule
\multicolumn{4}{l}{\textit{KL Penalty Configuration}} \\
\midrule
KL Schedule & $k^{-0.5}$ & - & Dynamic $\beta = k^{\text{schedule\_pow}}$ \\
Initial KL Penalty & 1.0 & - & Starting $\beta_0$ coefficient \\
Entropy Coefficient & 0.01 & - & Policy entropy regularization \\
\midrule
\multicolumn{4}{l}{\textit{Training Configuration}} \\
\midrule
Value Loss Coefficient & 0.5 & - & Critic loss weighting \\
PPO Epochs & 4 & - & Policy update iterations \\
Batch Size & 64 & samples & Mini-batch size \\
\midrule
\multicolumn{4}{l}{\textit{Architecture}} \\
\midrule
Hidden Dimensions & 64 & neurons & Network layer width \\
Activation & ReLU & - & Non-linear activation \\
Optimizer & Adam & - & Gradient descent \\
Package & Pytorch & - &\\
Device & CPU & - & Apple M4, Arm64 \\
\bottomrule
\end{tabular}
\end{table}

\end{document}